\documentclass[sigconf]{acmart}
\usepackage{booktabs} 
\usepackage{adjustbox}
\usepackage{multirow}
\usepackage{multicol}
\usepackage{bm}
\usepackage{pgfplots}
\usepackage{url}
\usepackage{hyperref}
\usepackage{subfigure}
\usepackage{tablefootnote}

\usepackage{CJKutf8}
\usepackage{amsmath}
\usepackage{amsfonts}
\usepackage{bm}
\usepackage{scalerel}

\pgfplotsset{width=10cm,compat=1.9}

\makeatletter
\def\BState{\State\hskip-\ALG@thistlm}
\makeatother
\usepackage{tkz-euclide}
\usetikzlibrary{shapes.arrows, fadings, automata,tikzmark,decorations.pathreplacing,patterns}
\usetikzlibrary{intersections}
\usetikzlibrary{arrows.meta}
\usetikzlibrary{positioning, fit, mindmap, trees, calc,tikzmark,shapes}

\newcommand{\headpfun}{\mathbf{P}_{H}}
\newcommand{\tailpfun}{\mathbf{P}_{T}}
\newcommand{\headradius}{R_{H}}
\newcommand{\tailradius}{R_{T}}
\newcommand{\headarea}{C_H}
\newcommand{\tailarea}{C_T}
\newcommand{\setofrel}{\mathcal{R}}
\AtBeginDocument{%
  \providecommand\BibTeX{{%
    \normalfont B\kern-0.5em{\scshape i\kern-0.25em b}\kern-0.8em\TeX}}}

\copyrightyear{2020}
\acmYear{2020}
\acmConference[KDD 2020]{KDD 2020: ACM SIGKDD Conference on KnowledgeDiscovery and Data Mining}{August 22--27, 2020}{San Diego, California USA}
\acmBooktitle{KDD 2020: ACM SIGKDD Conference on KnowledgeDiscovery and Data Mining,
  August 22--27, 2020, San Diego, California USA}
\acmPrice{15.00}
\acmPrice{15.00}
\acmISBN{978-1-4503-XXXX-X/18/06}

\acmSubmissionID{1166}

\usepackage[ruled,vlined]{algorithm2e}
\usepackage{xcolor}

\definecolor{tea_green}{RGB}{214, 234, 193}
\definecolor{hint_green}{RGB}{226,246,209}
\definecolor{Madang}{RGB}{190,235,159}
\definecolor{yellow_green}{RGB}{198,222,119}
\definecolor{link_water}{RGB}{221, 232, 250}
\definecolor{celestial_blue}{RGB}{52, 152, 219}
\definecolor{shakespeare}{RGB}{85, 154, 193}
\definecolor{buttermilk}{RGB}{255,242,174}
\definecolor{chardonnay}{RGB}{250,196,114}
\definecolor{rajah}{RGB}{253,180,98}
\definecolor{fog}{RGB}{213, 193, 234}
\definecolor{melon}{RGB}{254,191,181}
\definecolor{sundown}{RGB}{249, 180, 181}
\definecolor{mona_lisa}{RGB}{246,152,134}
\definecolor{salmon}{RGB}{242,131,107}

\definecolor{saltpan}{RGB}{238, 243, 232}
\definecolor{aqua_spring}{RGB}{232, 243, 232}
\definecolor{tea_green}{RGB}{214, 234, 193}
\definecolor{Madang}{RGB}{190,235,159}
\definecolor{fringy_flower}{RGB}{194, 234, 193}
\definecolor{aero_blue}{RGB}{193, 234, 213}
\definecolor{pixie_green}{RGB}{183,214,170}
\definecolor{french_pass}{RGB}{195,232,246}
\definecolor{ice_cold}{RGB}{169,232,220}
\definecolor{pale_turquoise}{RGB}{172,240,242}
\definecolor{cruise}{RGB}{179,226,205}
\definecolor{sail}{RGB}{163,205,235}
\definecolor{spindle}{RGB}{179,205,227}
\definecolor{link_water}{RGB}{221, 232, 250}
\definecolor{periwinkle}{RGB}{203,213,232}
\definecolor{zanah}{RGB}{220, 233, 213}
\definecolor{frostee}{RGB}{217, 231, 214}
\definecolor{opal}{RGB}{199, 221, 211}
\definecolor{jet_stream}{RGB}{188, 214, 210}
\definecolor{skeptic}{RGB}{153, 187, 167}
\definecolor{hint_green}{RGB}{226,246,209}
\definecolor{snow_flurry}{RGB}{230,245,201}
\definecolor{surf_crest}{RGB}{205,230,208}
\definecolor{yellow_green}{RGB}{198,222,119}
\definecolor{cream}{RGB}{255,255,204}
\definecolor{pale_prim}{RGB}{255,255,179}
\definecolor{spring_sun}{RGB}{242,243,195}
\definecolor{portafino}{RGB}{245,237,160}
\definecolor{buttermilk}{RGB}{255,242,174}
\definecolor{cream_brulee}{RGB}{255, 229, 151}
\definecolor{dairy_cream}{RGB}{254,226,189}
\definecolor{champagne}{RGB}{254,217,166}
\definecolor{chardonnay}{RGB}{250,196,114}
\definecolor{manhattan}{RGB}{226,180,125}
\definecolor{rajah}{RGB}{253,180,98}
\definecolor{early_dawn}{RGB}{252,243,218}
\definecolor{egg_shell}{RGB}{238, 234, 215}
\definecolor{selago}{RGB}{243, 232, 243}
\definecolor{quartz}{RGB}{219,223,238}
\definecolor{fog}{RGB}{213, 193, 234}
\definecolor{languid_lavender}{RGB}{222,203,228}
\definecolor{watusi}{RGB}{254,221,207}
\definecolor{coral_andy}{RGB}{243,204,205}
\definecolor{cosmos}{RGB}{248,209,210}
\definecolor{melon}{RGB}{254,191,181}
\definecolor{azalea}{RGB}{234, 193, 194}
\definecolor{beauty_bush}{RGB}{235, 185, 179}
\definecolor{sundown}{RGB}{249, 180, 181}
\definecolor{mona_lisa}{RGB}{246,152,134}
\definecolor{salmon}{RGB}{242,131,107}

\definecolor{summer_sky}{RGB}{58, 151, 233}
\definecolor{chateau_green}{RGB}{72, 179, 96}
\definecolor{matisse}{RGB}{25, 104, 167}
\definecolor{allports}{RGB}{31, 106, 125}
\definecolor{sun_shade}{RGB}{255, 144, 68}
\definecolor{flamingo}{RGB}{237, 88, 85}
\definecolor{studio}{RGB}{128, 91, 160}

\definecolor{maya_blue}{RGB}{102, 204, 255}
\definecolor{feijoa}{RGB}{178,223,138}
\definecolor{sushi}{RGB}{117, 168, 47}
\definecolor{norway}{RGB}{158, 194, 132}
\definecolor{japanese_laurel}{RGB}{53, 116, 40}
\definecolor{see_green}{RGB}{161,228,195}
\definecolor{monte_carlo}{RGB}{135,204,194}
\definecolor{granny_smith_apple}{RGB}{150,214,150}
\definecolor{moss_green}{RGB}{170,216,176}
\definecolor{chateau_green}{RGB}{72, 179, 96}
\definecolor{opal}{RGB}{164,207,190}
\definecolor{acapulco}{RGB}{117, 170, 148}
\definecolor{viridian}{RGB}{55, 137, 122}
\definecolor{amazon}{RGB}{56, 123, 84}
\definecolor{asparagus}{RGB}{123, 160, 91}
\definecolor{fruit_salad}{RGB}{91, 160, 94}
\definecolor{puerto_rico}{RGB}{72, 179, 150}
\definecolor{mountain_meadow}{RGB}{0, 163, 136}
\definecolor{matisse}{RGB}{25, 104, 167}
\definecolor{allports}{RGB}{31, 106, 125}
\definecolor{astral}{RGB}{55, 111, 137}
\definecolor{spring_leaves}{RGB}{46, 83, 117}
\definecolor{biscay}{RGB}{44, 62, 80}
\definecolor{midnight}{RGB}{0, 29, 50}
\definecolor{amethyst}{RGB}{153, 102, 204}
\definecolor{studio}{RGB}{128, 91, 160}
\definecolor{tapestry}{RGB}{194, 109, 132}
\definecolor{atomic_tangerine}{RGB}{255, 153, 102}
\definecolor{amber}{RGB}{255, 191, 0}
\definecolor{casablanca}{RGB}{244, 178, 84}
\definecolor{california}{RGB}{233, 140, 58}
\definecolor{tomato}{RGB}{255, 97, 56} 
\definecolor{alizarin}{RGB}{233, 58, 64}

\definecolor{linen}{RGB}{251, 239, 227}
\definecolor{double_pearl_lusta}{RGB}{253, 242, 208}
\definecolor{oasis}{RGB}{253, 242, 208}
\definecolor{milan}{RGB}{255, 254, 169}
\definecolor{texas}{RGB}{245, 232, 123}
\definecolor{maize}{RGB}{249, 212, 156}

\definecolor{turmeric}{RGB}{211, 178, 76}
\definecolor{saffron}{RGB}{249,193,62}
\definecolor{my_sin}{RGB}{255, 176, 59}
\definecolor{tree_poppy}{RGB}{246, 154, 27}
\definecolor{jaffa}{RGB}{240, 131, 58}
\definecolor{crusta}{RGB}{254, 127, 44}
\definecolor{tahiti_gold}{RGB}{223, 102, 36}
\definecolor{outrageous_orange}{RGB}{255, 100, 45}
\definecolor{safety_orange}{RGB}{254, 106, 0}

\definecolor{azalea}{RGB}{251, 196, 196}
\definecolor{oyster_pink}{RGB}{238,206,205} 
\definecolor{coral_candy}{RGB}{242,208,205} 
\definecolor{baby_pink}{RGB}{246, 194, 192}
\definecolor{petite_orchid}{RGB}{223, 157, 155}
\definecolor{apricot}{RGB}{241,140,122}
\definecolor{NY_pink}{RGB}{228,136,113}
\definecolor{carmine_pink}{RGB}{231, 76, 60}
\definecolor{deep_carmine_pink}{RGB}{236, 50, 67}

\definecolor{wewak}{RGB}{244, 143, 150}
\definecolor{light_coral}{RGB}{244, 127, 123}
\definecolor{bittersweet}{RGB}{255,111,105}
\definecolor{carnation}{RGB}{245, 80, 86}
\definecolor{flamingo}{RGB}{237, 88, 85}
\definecolor{sunset_orange}{RGB}{242,89,75}
\definecolor{ku_crimson}{RGB}{243, 0, 25}
\definecolor{amaranth}{RGB}{234,46,73}
\definecolor{valencia}{RGB}{214, 87, 70}
\definecolor{chilean_fire}{RGB}{215, 87, 44}
\definecolor{mexican_red}{RGB}{170, 41, 37}

\definecolor{napa}{RGB}{163, 154, 137}

\definecolor{athens_gray}{RGB}{236, 240, 241}
\definecolor{gallery}{RGB}{240,240,240}
\definecolor{mercury}{RGB}{230,230,230}
\definecolor{platinum}{RGB}{228,228,228}
\definecolor{silver}{RGB}{191,191,191}
\definecolor{aluminum}{RGB}{153,153,153}
\definecolor{ship_gray}{RGB}{77,77,77}
\definecolor{tuatara}{RGB}{67, 67, 67}

\definecolor{malibu}{RGB}{110, 180, 240}
\definecolor{celestial_blue}{RGB}{52, 152, 219}
\definecolor{curious_blue}{RGB}{41, 128, 185}
\definecolor{french_blue}{RGB}{0, 112, 182}
\definecolor{matisse}{RGB}{25, 104, 167}
\definecolor{shakespeare}{RGB}{85, 154, 193}
\definecolor{seagull}{RGB}{128,177,211}
\definecolor{jelly_bean}{RGB}{45, 126, 150}
\definecolor{venice_blue}{RGB}{87, 135, 105}
\definecolor{boston_blue}{RGB}{68, 147, 161}

\definecolor{turquoise}{RGB}{41,217,194}
\definecolor{java}{RGB}{2,190,196}
\definecolor{riptide}{RGB}{141,211,199}
\definecolor{mountain_meadow}{RGB}{0, 163, 136}
\definecolor{free_speech_aquamarine}{RGB}{0, 156, 114}

\definecolor{cosmic_latte}{RGB}{222, 247, 229}
\definecolor{chinook}{RGB}{163, 232, 178}
\definecolor{padua}{RGB}{121, 189, 143}
\definecolor{ocean_green}{RGB}{79, 176, 112}
\definecolor{pastel_green}{RGB}{107, 227, 135}
\definecolor{chateau_green}{RGB}{69, 191, 85}
\definecolor{RoyalBlue}{RGB}{69, 191, 85}
\definecolor{pigment_green}{RGB}{0, 175, 79}
\definecolor{fern}{RGB}{101,197,117}
\definecolor{killarney}{RGB}{56, 113, 66}

\definecolor{quartz}{RGB}{219,223,238}
\definecolor{spring_sun}{RGB}{242,243,195}
\definecolor{dairy_cream}{RGB}{254,226,189}
\definecolor{surf_crest}{RGB}{205,230,208}
\definecolor{french_pass}{RGB}{195,232,246}
\definecolor{cosmos}{RGB}{248,209,210}
\definecolor{portafino}{RGB}{245,237,160}
\definecolor{sail}{RGB}{163,205,235}
\definecolor{hint_green}{RGB}{226,246,209}
\definecolor{bittersweet}{RGB}{255,111,105}
\definecolor{java}{RGB}{2,190,196}
\definecolor{ice_cold}{RGB}{169,232,220}
\definecolor{bgc}{RGB}{245,245,245}
\definecolor{tuatara}{RGB}{67, 67, 67}
\definecolor{aluminum}{RGB}{153,153,153}
\definecolor{silver}{RGB}{191,191,191}
\definecolor{platinum}{RGB}{228,228,228}
\definecolor{mercury}{RGB}{230,230,230}
\definecolor{gallery}{RGB}{240,240,240}
\definecolor{free_speech_aquamarine}{RGB}{0, 156, 114}
\definecolor{sun_shade}{RGB}{255, 144, 68}
\definecolor{fern}{RGB}{101,197,117}
\definecolor{french_blue}{RGB}{0, 112, 182}
\definecolor{matisse}{RGB}{25, 104, 167}
\definecolor{sushi}{RGB}{117, 168, 47}
\definecolor{shakespeare}{RGB}{85, 154, 193}
\definecolor{egg_shell}{RGB}{238, 234, 215}
\definecolor{carnation}{RGB}{245, 80, 86}
\definecolor{flamingo}{RGB}{237, 88, 85}
\definecolor{jet_stream}{RGB}{188, 214, 210}
\definecolor{jelly_bean}{RGB}{45, 126, 150}
\definecolor{tree_poppy}{RGB}{246, 154, 27}
\definecolor{deep_carmine_pink}{RGB}{236, 50, 67}
\definecolor{copper_rust}{RGB}{155, 64, 74}
\definecolor{midnight}{RGB}{0, 29, 50}
\definecolor{chilean_fire}{RGB}{215, 87, 44}
\definecolor{puerto_rico}{RGB}{94, 194, 166}
\definecolor{japanese_laurel}{RGB}{53, 116, 40}
\definecolor{fire_engine_red}{RGB}{206, 37, 51}
\definecolor{ku_crimson}{RGB}{243, 0, 25}
\definecolor{turmeric}{RGB}{211, 178, 76}
\definecolor{tahiti_gold}{RGB}{223, 102, 36}
\definecolor{outrageous_orange}{RGB}{255, 100, 45}
\definecolor{crusta}{RGB}{254, 127, 44}
\definecolor{safety_orange}{RGB}{254, 106, 0}
\definecolor{pigment_green}{RGB}{0, 175, 79}
\definecolor{jaffa}{RGB}{240, 131, 58}
\definecolor{jet_stream}{rgb}{0.69,0.61,0.85}
\definecolor{jelly_bean}{rgb}{0.47,0.32,0.66}
\definecolor{azalea}{RGB}{251, 196, 196}
\definecolor{sundown}{RGB}{249, 180, 181}
\definecolor{light_coral}{RGB}{244, 127, 123}
\definecolor{wewak}{RGB}{244, 143, 150}

\definecolor{biscay}{RGB}{44, 62, 80}
\definecolor{carmine_pink}{RGB}{231, 76, 60}
\definecolor{athens_gray}{RGB}{236, 240, 241}
\definecolor{celestial_blue}{RGB}{52, 152, 219}
\definecolor{curious_blue}{RGB}{41, 128, 185}
\definecolor{my_sin}{RGB}{255, 176, 59}
\definecolor{viridian}{RGB}{70, 137, 102}
\definecolor{tomato}{RGB}{255, 97, 56}
\definecolor{mountain_meadow}{RGB}{0, 163, 136}
\definecolor{padua}{RGB}{121, 189, 143}
\definecolor{killarney}{RGB}{56, 113, 66}
\definecolor{ocean_green}{RGB}{79, 176, 112}
\definecolor{pastel_green}{RGB}{107, 227, 135}
\definecolor{chinook}{RGB}{163, 232, 178}
\definecolor{cosmic_latte}{RGB}{222, 247, 229}
\definecolor{chateau_green}{RGB}{69, 191, 85}
\definecolor{RoyalBlue}{RGB}{69, 191, 85}

\definecolor{blue0}{RGB}{240,249,232}
\definecolor{blue1}{RGB}{204,235,197}
\definecolor{blue2}{RGB}{168,221,181}
\definecolor{blue3}{RGB}{123,204,196}
\definecolor{blue4}{RGB}{78,179,211}
\definecolor{blue5}{RGB}{43,140,190}
\definecolor{blue6}{RGB}{8,88,158}

\definecolor{yellow0}{RGB}{255,255,212}
\definecolor{yellow1}{RGB}{254,227,145}
\definecolor{yellow2}{RGB}{254,196,79}
\definecolor{yellow3}{RGB}{254,153,41}
\definecolor{yellow4}{RGB}{236,112,20}
\definecolor{yellow5}{RGB}{204,76,2}
\definecolor{yellow6}{RGB}{140,45,4}

\begin{document}

\title{Exploiting Global Semantic Similarities in Knowledge Graphs by Relational Prototype Entities}




\author{Xueliang Wang, Jiajun Chen, Feng Wu, Jie Wang}
\affiliation{%
  \institution{University of Science and Technology of China}
  \city{Hefei}
  \state{Anhui}
  \country{China}
  }
\email{{xlwang95, jiajun98}@mail.ustc.edu.com,  	{fengwu, jiewangx}@ustc.edu.cn}





\begin{abstract}
Knowledge graph (KG) embedding aims at learning the latent representations for entities and relations of a KG in continuous vector spaces. An empirical observation is that the head (tail) entities connected by the same relation often share similar semantic attributes---specifically, they often belong to the same category---no matter how far away they are from each other in the KG; that is, they share global semantic similarities. However, many existing methods derive KG embeddings based on the local information, which fail to effectively capture such global semantic similarities among entities. To address this challenge, we propose a novel approach, which introduces a set of virtual nodes called \textit{\textbf{relational prototype entities}} to represent the prototypes of the head and tail entities connected by the same relations. By enforcing the entities' embeddings close to their associated prototypes' embeddings, our approach can effectively encourage the global semantic similarities of entities---that can be far away in the KG---connected by the same relation. Experiments on the entity alignment and KG completion tasks demonstrate that our approach significantly outperforms recent state-of-the-arts.
\end{abstract}





\maketitle

\section{Introduction}
Knowledge graphs (KGs) store a wealth of human knowledge in a structured way, which are usually collections of the factual triple: (\textit{head entity, relation, tail entity}). A large number of real-world KGs, such as Freebase \cite{freebase}, DBpedia \cite{dbpedia}, and YAGO \cite{yago}, have been created and successfully applied to a wide range of applications, such as information retrieval \cite{xiong2017explicit}, recommendation systems \cite{KGRS,kgat}, question answering \cite{KGQA}, etc.

Although commonly used KGs contain billions of triples, a single KG is usually far from complete and may not support these applications with sufficient facts \cite{rsn}. To tackle this issue, two KG tasks know as \textit{entity alignment} and \textit{KG completion} are proposed and have achieved massive attention recently \cite{transe,transd,rsn,ntn,mlp,convkb,gmnn,rdgcn,avrgcn}. Entity alignment, a.k.a., entity resolution or matching, aims to identify entities from different KGs that refer to the same real-world object. KG completion, a.k.a., link prediction, aims to automatically predict missing links between entities based on known links.

Inspired by word embedding \cite{mikolov2013distributed} that can well capture semantic meaning of words, researchers have paid increasing attention to \textit{knowledge graph embedding} \cite{transe,gcn-align,rotate,ptranse,chen2018co,rescal,hole,analogy}. The key idea of KG embedding is to embed entities and relations of a KG into continuous vector spaces, so as to simplify the manipulation while preserving the inherent structures of the KG \cite{KGEsurvey}.
For example, the distance-based models, such as TransE \cite{transe}, TransH \cite{transh}, TransR \cite{transr}, and RotatE \cite{rotate},  measure the plausibility of a fact as the distance between the two entities after an operation (e.g., translation or rotation) carried out by the relation. Moreover, the semantic matching models, such as RESCAL \cite{rescal}, DistMult \cite{distmult}, HolE \cite{hole}, ComplEx \cite{complex} and ConvE \cite{conve}, measure plausibility of facts by matching latent representations of entities and relations embodied in their vector spaces. More recently, the GCN-based models, such as R-GCN \cite{rgcn}, GCN-Align \cite{gcn-align}, MuGNN \cite{mugnn}, and AliNet \cite{AliNet}, are proposed to make use of the structure information in the KG, where the embedding of an entity is learned by recursively aggregating the embeddings of its neighbor entities. 

 \begin{figure*}[ht]
     \centering
     \includegraphics[scale=0.262]{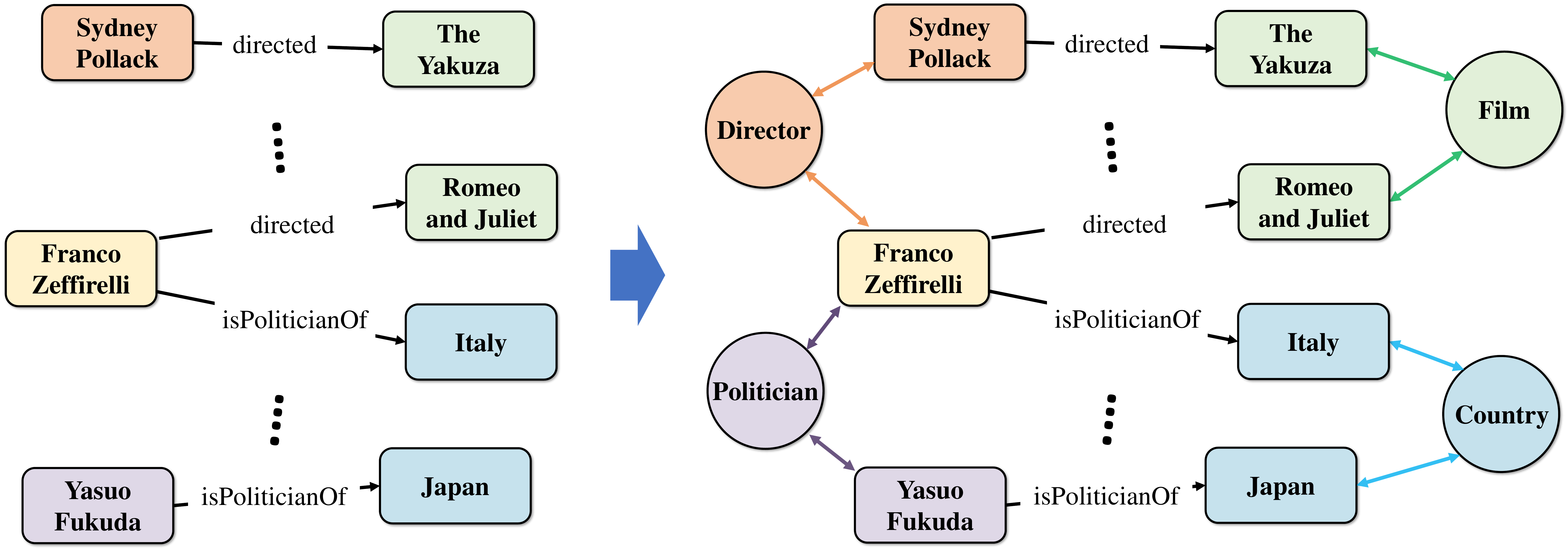}
     
     \caption{Illustration of our proposed method. The left is the original knowledge graph, and the right is the knowledge graph with relational prototype entities.
     In our method, we connect the head (tail) entities connected by the same relation to the associated head (tail) relational prototype entities. In the figure, the square blocks represent entities, and the round blocks represent relational prototype entities. 
     Notice that, the entity \textit{Franco Zeffirelli} is connected to two relational prototype entities at the same time.}
     
     \label{fig:figure1}
 \end{figure*}

While promising, many existing KG embedding models derive embeddings by using the local information of entities in the KG. They learn the embedding of an entity by aggregating the embeddings of its one-hop neighbors (e.g., TransE and RESCAL) or multi-hop neighbors (e.g., multi-layer GCNs).
An empirical observation is that the head (tail) entities connected by the same relation often share similar semantic attributes---specifically, they often belong to the same category---no matter how far away they are from each other in the KG; that is, they share global semantic similarities \cite{guo2015semantically,xie2016representation,krompass2015type,transt}.
For example, given three triples: (\textit{Sydney Pollack, directed, ?}), (\textit{Franco Zeffirelli,  directed, ?}), and (\textit{Yasuo Fukuda, isPoliticianOf, ?}),
the former two predicted entities should belong to the category of ``film'', thereby the embeddings of them should be more similar. 
However, many existing KG embedding methods do not explicitly model such global semantic similarities among entities, which results in that the embeddings are not expressive enough.

To tackle the aforementioned challenge, we innovatively introduce a set of virtual nodes called \textit{\textbf{relational prototype entities}}. 
Specifically, for each relation, we define a head (tail) category for the head (tail) entities connected by this relation; that is, if we have $N$ relations in the KG, we will define $N$ head categories and $N$ tail categories, respectively.
The proposed relational prototype entities represent the prototypes for each category.
We then connect the entities from the same category--no matter how far away they are from each other in the KG--to the associated  relational prototype entities. In this way, the proposed relational prototype entities can serve as the junctions to enhance the interaction of entities' embeddings from the same category.
Notice that, an entity may be connected to more than one relational prototype entity. 
For example, given two triples: (\textit{Franco Zeffirelli,  directed, Romeo and Juliet}) and (\textit{Franco Zeffirelli, isPoliticianOf, Italy}), the head entity \textit{Franco Zeffirelli}---which belongs to the categories of ``director'' and ``politician''---will be connected to the head relational prototype entities of \textit{directed} and  \textit{isPoliticianOf} at the same time. We show an illustration in Figure~\ref{fig:figure1}.

Our proposed method will enforce the entities' embeddings close to their associated prototypes' embeddings, so as to effectively encourage the global semantic similarities of entities---that can be far away in the KG---connected by the same relation. 
In this way, the embeddings in the KG can be more expressive, especially for the long-tail entities \cite{rsn} whose local information is not rich.
To show the effectiveness of our proposed method, we equip the popular GCN \cite{gcn} and RotatE \cite{rotate} models with our method and apply them to the entity alignment and KG completion tasks, respectively. 
Experiments show that our proposed method can not only effectively capture the global semantic similarities in the KG, but also significantly outperform recent state-of-the-art methods on the benchmark datasets.

\section{Notations and Background}\label{notations and background}
We first introduce the notations used in the paper in Section \ref{notations}. Then, we review the background of the entity alignment and KG completion tasks in Sections \ref{background entity alignment} and \ref{background KG completion}, respectively. 
\subsection{Notations}\label{notations}
In this paper, we formally represent a KG as $\mathcal{G}=(\mathcal{E},\mathcal{R},\mathcal{T})$, where $\mathcal{E}$ is the set of entities, $\mathcal{R}$ is the set of relations, and $\mathcal{T}\subset \mathcal{E}\times \mathcal{R} \times \mathcal{E}$ is the set of relational triples. We use the triple $(h,r,t)$ to denote a fact in the KG, where the lower-case letters $h$, $r$, and $t$ denote head entities, relations, and tail entities, respectively. The boldface lower-case letters $\textbf{h}$, $\textbf{r}$,
and $\textbf{t}$ represent the corresponding vectors, and we use $[\textbf{h}]_i$ to denote the $i$-th entry of the vector $\textbf{h}$. Let $k$ denote the the embedding dimension.
Given a set $\mathcal{S}$, we use $|\mathcal{S}|$ to denote the number of elements in $\mathcal{S}$.  We use $\lVert\cdot\rVert$ to denote the Euclidean norm. Let $\circ:\mathbb{R}^n\times\mathbb{R}^n\rightarrow\mathbb{R}^n$ denote the Hadamard product between two vectors, i.e.,
\begin{align}
  [\textbf{a}\circ \textbf{b}]_i=[\textbf{a}]_i\cdot [\textbf{b}]_i  
\end{align}


\subsection{Entity Alignment}\label{background entity alignment}
The entity alignment task aims to find entities from different KGs that refer to the same real-world object. Formally, given two KGs, i.e., $\mathcal{G}_1=(\mathcal{E}_1,\mathcal{R}_1,\mathcal{T}_1)$ and $\mathcal{G}_2=(\mathcal{E}_2,\mathcal{R}_2,\mathcal{T}_2)$, the goal of this task is to find alignment of entities between $\mathcal{G}_1$ and $\mathcal{G}_2$ based on the
partial pre-aligned entity pairs $\mathcal{A}^+=\{(i,j)\in \mathcal{E}_1\times \mathcal{E}_2|i\equiv j\}$, where $\equiv$ represents the alignment relationship. 

A critical  challenge of the entity alignment task lies in the heterogeneity of structures because counterpart entities in different KGs usually have dissimilar neighborhood structures \cite{pujara2013knowledge}.
Recently, the GCN-based embedding models have shown great potentials in dealing with the heterogeneity problem of the entity alignment task in both monolingual and cross-lingual scenarios \cite{gcn-align,AliNet,mugnn}. The GCN-based models update the embedding of an entity by aggregating the embeddings of its neighbors, and the similarity of entities is measured by the distance of entity embeddings.

A popular variant of GCNs is the multi-layer vanilla GCN proposed by \cite{gcn}. Every GCN encoder takes the hidden states of entity embeddings in the current layer as input, and outputs new entity embeddings with the following  layer-wise propagation rule
\begin{align}\label{gcn rule}
    \textbf{e}_i^{(l)}=\rho\left(\frac{\sum_{j\in \mathcal{N}_1(i)\cup\{i\}}\textbf{W}^{(l)}\textbf{e}_j^{(l-1)}}{| \mathcal{N}_1(i)\cup\{i\}|}\right),
\end{align}
where $\textbf{e}_i^{(l)}$ is the embedding of entity $i$ at the $l$-th layer, $\mathcal{N}_1(i)$ is the set of one-hop entity neighbors of entity $i$, $\textbf{W}^{(l)}$ is the trainable parameters at the $l$-th layer, and
$|\mathcal{N}_1(i)\cup\{i\}|$ is the number of elements in the set $\mathcal{N}_1(i)\cup\{i\}$. $\rho$ represents the activation function. 

GCN-Align \cite{gcn-align} uses the vanilla GCN \cite{gcn} to embed entities of different KGs into a unified vector space, and entity alignments are discovered based on the distances between entities in the embedding space. Other extensions including GMNN \cite{gmnn}, RDGCN \cite{rdgcn} and AVR-GCN \cite{avrgcn}. More recently, MuGNN \cite{mugnn} proposes a two-step method of rule-based KG completion and multi-channel GNNs for entity alignment. The learned rules rely on relation alignment to resolve schema heterogeneity. AliNet \cite{AliNet} proposes gated multi-hop neighborhood aggregation to mitigate the non-isomorphism of neighborhood structures in KGs. However, the most aforementioned models focus on the local structure information (e.g., one-hop or multi-hop neighbors), which do not explicitly consider the global semantic similarities among entities mentioned above.

\subsection{KG Completion}\label{background KG completion}
The KG completion task aims to complete the missing facts in a single KG. Formally, given a KG $\mathcal{G}=(\mathcal{E},\mathcal{R},\mathcal{T})$, it aims to predict the head entity $h$ given $(?, r, t)$ or predict the tail entity $t$ given $(h, r, ?)$. To tackle this problem, a general approach is to define a score function $f_r(\textbf{h},\textbf{t})$ to measure the plausibility of each fact $(h,r,t)$. The goal of the optimization is to score true facts $(h,r,t)$ observed in the KG higher than false facts $(h',r,t)$ or $(h,r,t')$.

 A straightforward method to tackle the KG completion task is to model $f_r(\textbf{h},\textbf{t})$ as a distance-based function. For example, the popular TransE \cite{transe} and its extentions, such as TransH \cite{transh}, TransR \cite{transr}, TransD \cite{transd}, and TranSparse \cite{transparse}, model relations as translations. More recently, the RotatE model \cite{rotate} defines each relation as a rotation instead of a translation from the head entity to the tail entity in the complex vector space. The score function $f_r(\textbf{h},\textbf{t})$ of RotatE is defined as
\begin{align}\label{rotate function}
    f_r(\textbf{h},\textbf{t})=-\|\textbf{h}\circ\textbf{r}-\textbf{t}\|,
\end{align}
where $\textbf{h}$, $\textbf{r}$, $\textbf{t}\in \mathbb{C}^k$, and the modulus of $[\textbf{r}]_i$ is 1. Compared to the existing translational models, RotatE can better model symmetry relations \cite{rotate}. Notice that, the aforementioned methods are triple-level learning methods \cite{rsn}, and most of them only consider the one-hop neighbors and do not explicitly model the global semantic similarities that the head and tail entities connected by the same relation often belong to the same category, respectively.

\section{Methods} \label{methods}
As mentioned above, many knowledge graph embedding methods mainly focus on using local information in the KG, which do not explicitly consider the global semantic similarities among entities. 
Inspired by the empirical observation that the head (tail) entities connected by the same relation often share similar semantic attributes, 
we innovatively introduce a set of relational prototype entities $P(r)$, and we use ${\mathbf{P}(r)}$ to represent the corresponding vector. Specifically, for each relation $r$, we define a head relational prototype entity $P_{H}(r)$ and a tail relational prototype entity $P_{T}(r)$ to serve as the prototypes for the head and tail categories of each relation. $\headpfun(r)$ and $\tailpfun(r)$ are the corresponding vectors of $P_{H}(r)$ and $P_{T}(r)$. Then, we add the head and tail relational prototype entities to the KG $\mathcal{G}=(\mathcal{E},\mathcal{R},\mathcal{T})$ in a simple yet effective way. We connect the entities from the same category to the associated  relational prototype entities. 
Thereby, the proposed relational prototype entities can serve as the junctions to enhance the interaction of entity embeddings from the same category.  We show a simple illustration of our method in Figure~\ref{illustration2}.

The entity alignment and KG completion are two fundamental tasks in KG embedding \cite{rsn}. To show the effectiveness and universality of our proposed approach, we equip the popular GCN \cite{gcn} and RotatE \cite{rotate} models with relational prototype entities and apply them to the entity alignment and KG completion tasks, respectively. In Sections \ref{method entity alignment} and \ref{method kg completion}, we detail the proposed methods for the two tasks. In Section \ref{analysis}, we analyze the proposed methods.

\subsection{Entity Alignment}\label{method entity alignment}
As introduced in Section \ref{background entity alignment}, the GCN-based embedding models \cite{gcn-align,rdgcn,avrgcn,mugnn,AliNet} have achieved great success on the entity alignment task recently. Therefore, we use the vanilla GCN \cite{gcn} introduced in Section \ref{background entity alignment} as our baseline model. The baseline GCN gets the embedding of entity $i$ by aggregating the embeddings of its one-hop entity neighbors by \eqref{gcn rule}.
In our proposed method, we add relational prototype entities to the KG. Thus, we redefine the propagation rule in \eqref{gcn rule} as
\begin{align}
    &\textbf{e}_i^{(l)}\hspace{-1mm}=\hspace{-1mm}\rho\hspace{-0.6mm}\left(\frac{\lambda\hspace{-0.6mm}\sum_{j\in \mathcal{N}_1(i)\cup\{i\}}\textbf{W}^{(l)}\textbf{e}_j^{(l\hspace{-0.3mm}-\hspace{-0.3mm}1)}\hspace{-1mm}+\hspace{-0.8mm}(1\hspace{-0.6mm}-\hspace{-0.6mm}\lambda)\sum_{P(r)\in\mathcal{N}_1^R(i)}\textbf{W}^{(l)}{\mathbf{P}^{(l\hspace{-0.3mm}-\hspace{-0.3mm}1)}(r)}}{\lambda|\mathcal{N}_1(i)\cup\{i\}|+(1-\lambda)|\mathcal{N}_1^R(i)|}\right),\label{gcn rule new1}\\
    &\mathbf{P}^{(l)}(r)=\rho\left(\frac{\lambda\sum_{j\in \mathcal{N}_1(P(r))}\textbf{W}^{(l)}\textbf{e}_j^{(l-1)}+(1-\lambda)\mathbf{W}^{(l)}\mathbf{P}^{(l-1)}(r)}{\lambda|\mathcal{N}_1(P(r))|+1-\lambda}\right)\label{gcn rule new2},
\end{align}
where $\mathcal{N}_1(i)$ is the set of one-hop entity neighbors of entity $i$, $\mathcal{N}_1^R(i)$ is set of relational prototype entities directly connected to the entity $i$, and $\mathcal{N}_1(P(r))$ is the set of one-hop entity neighbors of the relational prototype entity $P(r)$. $0<\lambda\le1$ is the hyperparameter that balances the weight. Notice that, when $\lambda=1$, \eqref{gcn rule new1} is the same as \eqref{gcn rule}; that is, the relational prototype entities will not affect the embeddings of entities in the original KG. 
In our method, the entities connected by the same relation can share global semantic similarities through relational prototype entities, and we call this modified model \textbf{RPE-GCN}.

Many GCN-based methods use the hidden outputs at the last layer as the final embedding of each entity, i.e., $\textbf{e}_i=\textbf{e}_i^{(L)}$, where $L$ is the number of GCN layers. However, the hidden representations of all layers could contribute to propagating alignment information \cite{AliNet}. Therefore, we get the embedding of each entity by aggregating the hidden states of all the layers following \cite{AliNet}, i.e.,
\begin{align}\label{aggragation}
    \textbf{e}_i=\frac{1}{L}\sum_{l=1}^{L}\textbf{e}_i^{(l)}.
\end{align}

\textbf{Loss Function}. On the entity alignment task, we usually judge whether two entities are equivalent by the distance between them. We want the embeddings of aligned entities to have a small distance while those of unaligned entities have a larger distance. Following \cite{mugnn}, the loss function is given as below
\begin{align}\label{loss function entity alignment}
    \mathcal{L}=\sum_{(i,j)\in \mathcal{A}^+}\sum_{(i',j')\in \mathcal{A}^-}[\|\textbf{e}_i-\textbf{e}_j\|+\gamma-\|\textbf{e}_{i'}-\textbf{e}_{j'}\|]_+,
\end{align}
where $\mathcal{A}^-$ is the set of negative entity alignment pairs of $\mathcal{A}^+$, $[\cdot]_+={\rm max}\{0,\cdot\}$ is the maximum between 0 and the input, and $\gamma>0$ is the margin hyperparameter separating positive and negative entity alignments. Following \cite{mugnn,sun2018bootstrapping}, we select 25 entities closest to the corresponding entity  as negative samples by calculating cosine similarity in the same KG, and negative samples will be recalculated every 5 epochs. 

\begin{figure}
  \centering
  \includegraphics[width=0.95\linewidth]{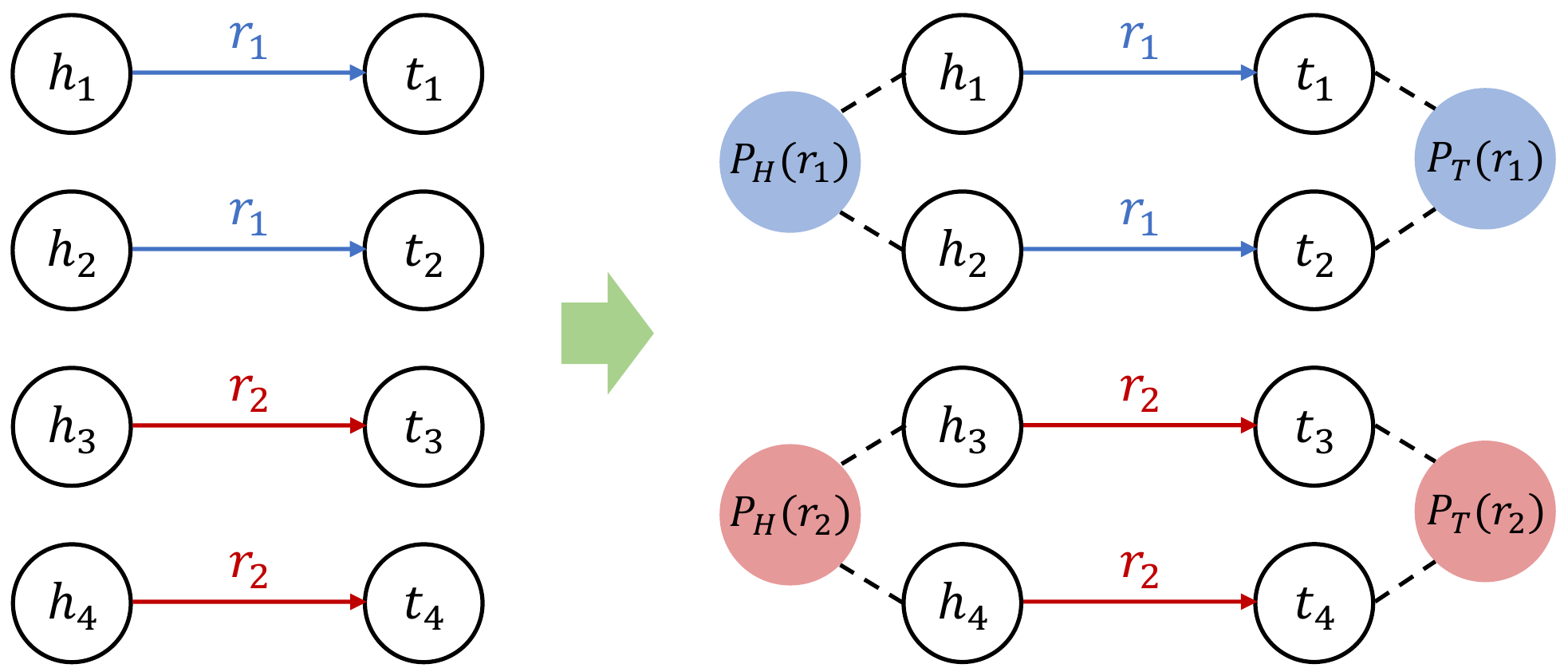}
  \caption{A simple illustration of our proposed method. $P_{H}(r_1)$ and $P_{T}(r_1)$ represent the head and tail relational prototype entities of relation $r_1$, respectively. $P_{H}(r_2)$ and $P_{T}(r_2)$ represent the head and tail relational prototype entities of relation $r_2$, respectively.}
  \label{illustration2}
\end{figure}

\subsection{KG Completion}\label{method kg completion}
As discussed in Section \ref{background KG completion}, many KG completion methods are triple-level learning methods \cite{rsn}, which fail to explicitly model the global semantic similarities among entities. To tackle the limitations, we redefine the entity embeddings by incorporating the embeddings of the associated relational prototype entities. We reformulate the score function used in the KG completion task as $f_r(g(\textbf{h},\headpfun(r)),\,g(\textbf{t},\tailpfun(r)))$, where $g(\cdot)$ is the aggregation function. 

In this paper, we choose the popular RotatE \cite{rotate} as the baseline model in the experiments. We then redefine the score function in \eqref{rotate function} as:
\begin{align*}
    f_r(&g(\textbf{h},\headpfun(r)),\,g(\textbf{t},\tailpfun(r)))=\\
    &-\|(\lambda\textbf{h}+(1-\lambda)\headpfun(r))\circ\textbf{r}-(\lambda\textbf{t}+(1-\lambda)\tailpfun(r))\|,
\end{align*}
where $0<\lambda\le1$ is the hyperparameter to balance the weight, and we choose the sum operation as the aggregation function. We call the modified model \textbf{RPE-RotatE}.

\textbf{Loss Function.} To train the model, we use the negative sampling loss functions with self-adversarial training \cite{rotate}:
\begin{align*}
    L=&-\log\sigma(\gamma-f_r(g(\textbf{h},\headpfun(r)),\,g(\textbf{t},\tailpfun(r))))\\&-\sum_{i=1}^np(h'_i,r,t'_i)\log\sigma(f_r(g(\textbf{h}_i',\headpfun(r)),\,g(\textbf{t}_i',\tailpfun(r)))-\gamma),
\end{align*}
where $\gamma>0$ is a fixed margin, $\sigma$ is the sigmoid function, and $(h'_i,r,t'_i)$ is the $i$-th negative triple. Moreover,
\begin{align*}
    p(h'_j,r,t'_j|\{(h_i,r_i,t_i)\})=\frac{\exp \alpha f_r(g(\textbf{h}_j',\headpfun(r)),\,g(\textbf{t}_j',\tailpfun(r)))}{\sum_i \exp \alpha f_r(g(\textbf{h}_i',\headpfun(r)),\,g(\textbf{t}_i',\tailpfun(r)))}
\end{align*}
is the probability distribution of sampling negative triples, where $\alpha$ is the temperature of sampling.

\subsection{Analysis}\label{analysis}
This section provides the theoretical analysis of our proposed method. We take the KG completion task as an example. We first prove that our method can enforce the entities’ embeddings close to their associated prototypes’ embeddings. Then, we prove that such property will benefit the performance.
For the analysis of the entity alignment task, please refer to the Appendix \ref{analysis appendix}. 

We begin with the assumption that the embeddings of different relational prototype entities are different. We then define the head and tail relational prototype areas $\headarea(r)$ and $\tailarea(r)$ as
\begin{align*}
&\headarea(r) = \{\textbf{x} \,|\, \lVert \textbf{x}-\headpfun(r) \rVert \le \headradius(r) \},\\
&\tailarea(r) = \{\textbf{x} \,|\, \lVert \textbf{x}-\tailpfun(r) \rVert \le \tailradius(r) \},
\end{align*}
where $\headradius(r)$ and $\tailradius(r)$ are given by
\begin{align*}
    &\headradius(r)={\rm max}\{\|\textbf{h}-\headpfun(r)\|: (h,r,t)\in \mathcal{T}\},\\
    &\tailradius(r)={\rm max}\{\|\textbf{t}-\tailpfun(r)\|: (h,r,t)\in \mathcal{T}\}.
\end{align*}
The set of relational prototype areas is defined as 
\begin{align*}
\mathcal{C} = \cup_{r\in\setofrel}(\{\headarea(r)\} \cup \{\tailarea(r)\}),
\end{align*}
and the distance between $C_1, C_2\in \mathcal{C}$ is
\begin{align*}
d(C_1, C_2) &= \inf\{\lVert \textbf{x}_1 - \textbf{x}_2\rVert: \textbf{x}_1\in C_1,\textbf{x}_2\in C_2  \}.
\end{align*}

In our method, we redefine the embedding of an entity by aggregating the embeddings of its associated  relational prototype entity and itself. For the triple $(h, r, t)$, the aggregation operation $g(\cdot)$ is
\begin{align*}
    \widehat{\textbf{h}} = g(\textbf{h}, \headpfun(r)) = \lambda \textbf{h} + (1-\lambda) \headpfun(r), 
\end{align*}
where $\widehat{\textbf{h}}$ denotes the redefined embedding of $\textbf{h}$, and $0 < \lambda \le 1$ is the hyperparameter. It is clear that the distance between  $\widehat{\textbf{h}}$ and $\headpfun(r)$ becomes smaller than that between $\textbf{h}$ and $\headpfun(r)$, i.e.,
\begin{align}\label{eq_scale}
\lVert \widehat{\textbf{h}} - \headpfun(r) \rVert = \lambda \lVert \textbf{h} - \headpfun(r)\rVert \le \lVert \textbf{h} - \headpfun(r) \rVert, 
\end{align}
and the equal sign holds only when $\textbf{h}=\headpfun(r)$ or $\lambda=1$.

The radii of the head and tail relational prototype areas after aggregation are
\begin{align*}
    &\widehat{\headradius}(r) = {\rm max}\{\lambda\|\textbf{h}-\headpfun(r)\|: (h,r,t)\in \mathcal{T}\} = \lambda \headradius(r),\\
        &\widehat{\tailradius}(r) = {\rm max}\{\lambda\|\textbf{t}-\tailpfun(r)\|: (h,r,t)\in \mathcal{T}\} = \lambda \tailradius(r) .
\end{align*}
Thus, we can conclude that every relational prototype area shrinks after the aggregation operation when $0 < \lambda<1$. If the value of $\lambda$ is small enough, there will be no overlap between any two different relational prototype areas in $\mathcal{C}$.

Next, we will prove that such property benefits the performance on the KG completion task. We begin with the assumption:
 $$f_r(\headpfun(r),\tailpfun(r))=-\lVert\headpfun(r) \circ \textbf{r} - \tailpfun(r)\rVert = 0 \Rightarrow \headpfun(r) \circ \textbf{r} = \tailpfun(r).$$

\begin{lemma}\label{lemma_score_function}
For a relation $r$, the following inequalities hold 
\begin{align*}
    f_r(\textbf{h},\textbf{t}) &\ge -\lVert \textbf{h} -\headpfun(r) \rVert - \lVert\textbf{t}-\tailpfun(r) \rVert,\\
    f_r(\textbf{h},\textbf{t}) &\le  \lVert\textbf{t}-\tailpfun(r) \rVert - \lVert \textbf{h} -\headpfun(r) \rVert,\\
    f_r(\textbf{h}, \textbf{t}) &\le \lVert \textbf{h} -\headpfun(r) \rVert - \lVert\textbf{t}-\tailpfun(r) \rVert.
\end{align*}
\end{lemma}
\begin{proof}
Based on the assumption $\headpfun(r) \circ \textbf{r} = \tailpfun(r)$, we can get
\begin{align*}
    f_r(\textbf{h}, \textbf{t}) &= -\lVert \textbf{h}\circ \textbf{r} - \textbf{t}\rVert = -\lVert (\textbf{h}-\headpfun(r))\circ \textbf{r}-(\textbf{t}-\tailpfun(r)) \rVert \\
    &\ge -\lVert (\textbf{h} -\headpfun(r))\circ \textbf{r} \rVert - \lVert\textbf{t}-\tailpfun(r) \rVert. 
\end{align*}
As the modulus of $[\textbf{r}]_i$ is 1, we have
\begin{align*}
    \lVert (\textbf{h} -\headpfun(r))\circ \textbf{r} \rVert = \lVert \textbf{h} -\headpfun(r)\rVert.
\end{align*}
Then, we can get
$$
f_r(\textbf{h},\textbf{t}) \ge -\lVert \textbf{h} -\headpfun(r) \rVert - \lVert\textbf{t}-\tailpfun(r) \rVert.
$$
We can prove the other two inequalities similarly. Thus, the lemma is proved.
\end{proof}

Based on the conclusion that there will be no overlap between any two different relational prototype areas in $\mathcal{C}$ if $\lambda$ is small enough, we can get the following two theorems. Notice that, for convenience, $\textbf{h}$ and $\textbf{t}$ used in the following theorems represent the entity embeddings after the aggregation operation.  

\begin{theorem}\label{theorem_kgc_1}
For a relation $r$, if $d(C, \headarea(r)) > 2 \tailradius(r)$ for any $C\in \mathcal{C}\setminus\{\headarea(r)\}$, we have $f_r(\textbf{h}_1, \textbf{t}) > f_r(\textbf{h}_2, \textbf{t}) $ for any $\textbf{h}_1\in \headarea(r), \textbf{h}_2\notin \headarea(r)$, $\textbf{t}\in\tailarea(r)$.
\end{theorem}

\begin{proof}
As $\textbf{h}_2\notin C_H(r)$, there exists a $C_{h2}\in \mathcal{C}\setminus\{\headarea(r)\}$ such that $\textbf{h}_2\in C_{h2}$. Then, we have 
\begin{align*}
    \lVert\textbf{h}_2 - \headpfun(r) \rVert &\ge d(C_{h2}, \headarea(r)) + \headradius(r) > \headradius(r) + 2\tailradius(r).
\end{align*}
Based on Lemma \ref{lemma_score_function}, we can get
\begin{align*}
    f_r(\textbf{h}_1, \textbf{t}) &\ge -\lVert \textbf{h}_1 -\headpfun(r) \rVert - \lVert\textbf{t}-\tailpfun(r) \rVert \ge -\headradius(r) - \tailradius(r),\\
    f_r(\textbf{h}_2, \textbf{t}) &\le \lVert\textbf{t}-\tailpfun(r) \rVert - \lVert \textbf{h}_2 -\headpfun(r) \rVert < -\headradius(r) - \tailradius(r).
\end{align*}
Therefore, we have $f_r(\textbf{h}_1, \textbf{t}) > f_r(\textbf{h}_2, \textbf{t})$, which completes the proof.
\end{proof}

\begin{theorem}\label{theorem_kgc_2}
For a relation $r$, if $d(C, \tailarea(r)) > 2 \headradius(r)$ for any $C\in \mathcal{C}\setminus\{\tailarea(r)\}$, we have $f_r(\textbf{h}, \textbf{t}_1) > f_r(\textbf{h}, \textbf{t}_2) $ for any $\textbf{h}\in \headarea(r), \textbf{t}_1\in \tailarea(r), \textbf{t}_2\notin\tailarea(r)$.
\end{theorem}

\begin{proof}
We omit the proof of Theorem \ref{theorem_kgc_2} since it is similar to the proof of Theorem \ref{theorem_kgc_1}. 
\end{proof}

From Theorem \ref{theorem_kgc_1} and Theorem \ref{theorem_kgc_2}, we can see that the score of entities belonging to the category of the correct answer will be higher than entities from other categories. 

In conclusion, our method can enforce the entities’ embeddings close to their associated prototypes’  embeddings, so as to encourage the returned entities with high ranking to all belong to the category of the correct answer.
For example, when predicting the director of a film, our method will encourage the score of entities belonging to the category of ``director'' to be higher than other entities, which is beneficial to improving the performance.

\begin{table}[ht]
  \caption{Statistics of the number of relations, entities, and triples for DBP15k and DWY100k datasets.}
  \label{tab: datasets of entity alignment}
  \begin{tabular}{lccc}
    \toprule
    Datastes&\#Relation&\#Entity&\#Triple\\
    \midrule
    ${\rm DBP_{ZH}}$ & 2,830& 66,469& 153,929\\
    ${\rm DBP_{EN}}$ & 2,317& 98,125& 237,674\\
    \midrule
    ${\rm DBP_{JA}}$ & 2,043 &65,744 &164,373\\
    ${\rm DBP_{EN}}$ & 2,096 &95,680 &233,319\\
    \midrule
    ${\rm DBP_{FR}}$ & 1,379 &66,858 &192,191\\
    ${\rm DBP_{EN}}$ & 2,209 &105,889 &278,590\\
    \midrule
    ${\rm DWY_{DB}}$ & 330 &100,000 &463,294\\
    ${\rm DWY_{WD}}$ & 220 &100,000 &448,774\\
    \midrule
    ${\rm DWY_{DB}}$ & 302 &100,000 &428,952\\
    ${\rm DWY_{YG}}$ & 31 &100,000 &502,563\\
  \bottomrule
\end{tabular}
\end{table}

\section{Experiments and Results}\label{experiments}
In this section, we evaluate the proposed methods on two representative KG embedding tasks: entity alignment and KG completion. For each task, we conduct experiments on several commonly used real-world datasets, and report the results compared with several state-of-the-art methods.

\begin{table*}[ht]
    \caption{Entity alignment results on DBP15k and DWY100k. Results of MTransE, JAPE, AlignE, GCN-Align, and MuGNN are taken from \cite{mugnn}. Other results are taken from \cite{AliNet}. The best results are in bold, and the second best results are underlined.}.
    \centering
    \resizebox{2.1\columnwidth}!{\begin{tabular}{lc c c c   c c c c  c c c c ccc }
        \toprule
          Methods&\multicolumn{3}{c}{\textbf{${\rm DBP_{ZH-EN}}$}}&  \multicolumn{3}{c}{\textbf{${\rm DBP_{JA-EN}}$}} & \multicolumn{3}{c}{\textbf{${\rm DBP_{FR-EN}}$}} &
          \multicolumn{3}{c}{\textbf{${\rm DBP-WD}$}} &
          \multicolumn{3}{c}{\textbf{${\rm DBP-YG}$}}\\
         \cmidrule(lr){2-4}
         \cmidrule(lr){5-7}
         \cmidrule(lr){8-10}
         \cmidrule(lr){11-13}
         \cmidrule(lr){14-16}
         & H@1 & H@10 & MRR & H@1 & H@10 & MRR  & H@1 & H@10 & MRR & H@1 & H@10& MRR & H@1 & H@10& MRR\\
        \midrule
        MTransE \cite{mtranse}& 0.308 &0.614 &0.364 &0.279 &0.575 &0.349 &0.244 &0.556 &0.335 &0.281 &0.520 &0.363 &0.252 &0.493 &0.334\\
        IPTransE \cite{iptranse}& 0.406 &0.735 &0.516 &0.367 &0.693 &0.474 &0.333 &0.685 &0.451 &0.349 &0.638 &0.447 &0.297 &0.558 &0.386\\
        JAPE \cite{sun2017cross}& 0.412 &0.745 &0.490 &0.363 &0.685 &0.476 &0.324 &0.667 &0.430 &0.318 &0.589 &0.411 &0.236 &0.484 &0.320\\
        AlignE \cite{sun2018bootstrapping}&0.472 &0.792 &0.581 &0.448 &0.789 &0.563 &0.481 &0.824 &0.599 &0.566 &0.827 &0.655 &0.633 &0.848 &0.707\\
        GCN-Align \cite{gcn-align}& 0.413 &0.744 &0.549 &0.399 &0.745 &0.546 &0.373 &0.745 &0.532 &0.506 &0.772 &0.600 &0.597 &0.838 &0.682\\
        SEA \cite{sea}& 0.424 &0.796 &0.548 &0.385 &0.783 &0.518 &0.400 &0.797 &0.533 &0.518 &0.802 &0.616 &0.516 &0.736 &0.592\\
        RSN \cite{rsn}& 0.508 &0.745 &0.591 &0.507 &0.737 &0.590 &0.516 &0.768 &0.605 &0.607 &0.793 &0.673 &0.689 &0.878 &0.756\\
        MuGNN \cite{mugnn}& 0.494 &\underline{0.844} &0.611 &0.501 &\underline{0.857} &0.621 &0.495 &\underline{0.870} &0.621 &0.616 &0.897 &0.714 &0.741 &0.937 &0.810\\
        AliNet \cite{AliNet}&\underline{0.539} &0.826 &\underline{0.628} &\underline{0.549} &0.831 &\underline{0.645} &\underline{0.552} &0.852 &\underline{0.657} &\textbf{0.690} &\underline{0.908} &\textbf{0.766} &\underline{0.786} &\underline{0.943} &\underline{0.841}\\
        \midrule
        GCN \cite{gcn}& 0.477&0.828&0.593&0.495&0.848&0.613&0.498&0.861&0.619&0.611    &0.891    &0.707    &0.747    &0.939    &0.814\\
        RPE-GCN&\textbf{0.576}&\textbf{0.878}    &\textbf{0.678}    &\textbf{0.582}    &\textbf{0.883}    &\textbf{0.685}    &\textbf{0.597}    &\textbf{0.899}    &\textbf{0.701}&\underline{0.678}    &\textbf{0.915}    &\underline{0.762}    &\textbf{0.797}    &\textbf{0.957}    &\textbf{0.854}\\
        \bottomrule
    \end{tabular}}
    \label{results:entity alignment}
\end{table*}

\begin{table*}[ht]
    \label{table: dbi}
    \caption{KG completion results on WN18RR, FB15k-237, and YAGO3-10. Results of TransE and RotatE are taken from \cite{convkb} and \cite{rotate}, respectively. Other results are taken from \cite{conve}. The best results are in bold.}
    \centering
    \begin{tabular}{l  c c c c  c c c c  c c c c }
        \toprule
          Methods&\multicolumn{4}{c}{{WN18RR}}&  \multicolumn{4}{c}{{FB15k-237}} & \multicolumn{4}{c}{{YAGO3-10}}\\
         \cmidrule(lr){2-5}
         \cmidrule(lr){6-9}
         \cmidrule(lr){10-13}
         & MRR & H@1 & H@3 & H@10 & MRR & H@1 & H@3 & H@10 & MRR & H@1 & H@3 & H@10 \\
        \midrule
        TransE \cite{transe} & 0.226 &   -  &   -  & 0.501 & 0.294 &   -  &   -  & 0.465 & - & - & - & -\\
        DistMult \cite{distmult} & 0.43 & 0.39 & 0.44 & 0.49 & 0.241 & 0.155 & 0.263 & 0.419 & 0.34 & 0.24 & 0.38 & 0.54 \\
        ConvE \cite{conve}& 0.43  & 0.40 & 0.44 & 0.52 & 0.325 & 0.237 & 0.356 & 0.501 & 0.44  & 0.35  & 0.49  & 0.62\\
        ComplEx \cite{complex}& 0.44  & 0.41  & 0.46  & 0.51  & 0.247 & 0.158 & 0.275 & 0.428 & 0.36 & 0.26 & 0.40 & 0.55\\
        \midrule
        RotatE \cite{rotate}& 0.476 & 0.428 & 0.492 & 0.571 & 0.338 & 0.241 & 0.375 & 0.533 & 0.495 & 0.402 & 0.550 & 0.670\\
        RPE-RotatE &\textbf{0.501	}&\textbf{0.457}&\textbf{	0.515}&\textbf{	0.585 }&\textbf{0.355	}&\textbf{0.261	}&\textbf{0.390}&\textbf{	0.546 }&\textbf{0.546}&\textbf{	0.460}&\textbf{	0.598}&\textbf{	0.703}\\
        \bottomrule
    \end{tabular}
    \label{results:link prediction}
\end{table*}

\subsection{Entity Alignment Results}
\subsubsection{\textbf{Entity alignment datasets}} 

We conduct the entity alignment experiments on the DBP15k \cite{sun2017cross} and DWY100k \cite{sun2018bootstrapping} datasets following the latest progress \cite{sun2017cross,sun2018bootstrapping,mugnn,AliNet}. DBP15k contains three datasets built
from multi-lingual DBpedia: ${\rm DBP_{ZH-EN}}$ (Chinese to English), ${\rm DBP_{JA-EN}}$ (Japanese to English), and ${\rm DBP_{FR-EN}}$ (French to English). All the above datasets include 15,000 entity pairs as seed alignments. DWY100k consists of two large-scale cross-resource datasets: DBP-WD (DBpedia to Wikidata) and DBP-YG (DBpedia to YAGO3). Each dataset has 100,000 entity pairs as seed alignments. Following the previous work \cite{mugnn,AliNet}, we split 30\% of the entity seed alignments as training data, and leave the remaining data for the test. We list the statistics of the two datasets in Table \ref{tab: datasets of entity alignment}.

\subsubsection{\textbf{Implementation Details}} 
In the experiments, the hyperparameters of GCN (baseline) are taken from \cite{mugnn}. To make a fair comparison, we set the hyperparameters of RPE-GCN and the baseline model to the same. We set the embedding size to 128. We stack two layers (L=2) of GCN. We set $\gamma=1.0$ and $\lambda=0.5$. More details about the hyperparameters are listed in the Appendix \ref{hyperparameters}.

By convention \cite{gcn-align,mugnn,AliNet}, we choose Mean Reciprocal Rank (MRR) and Hits at N (H@N) as the evaluation metrics. MRR is the average of the reciprocal of the rank results. H@N indicates the percentage of the targets that have been correctly ranked in top N. Higher MRR or H@N indicates better performance.

\subsubsection{\textbf{Main Results}} 
In the experiments, we choose several state-of-the-art embedding-based entity alignment methods for comparison: MTransE \cite{mtranse}, IPTransE \cite{iptranse}, JAPE \cite{sun2017cross}, AlignE \cite{sun2018bootstrapping}, GCN-Align\cite{gcn-align}, SEA \cite{sea}, RSN \cite{rsn}, MuGNN \cite{mugnn}, AliNet \cite{AliNet}, and the baseline GCN \cite{gcn}. As claimed by \cite{AliNet}, the recent GNN-based models like GMNN \cite{gmnn} and RDGCN \cite{rdgcn} incorporate the surface information of entities into the embeddings. As our model mainly relies on structure and semantic information in the KG, we do not take these models into comparison following \cite{AliNet}.

We present the entity alignment results in Table \ref{results:entity alignment}. 
 It is evident that the proposed RPE-GCN achieves much better results than the baseline GCN. Specifically, on the ${\rm DBP15k}$ datasets, RPE-GCN gains about 10\% higher H@1, 4\% higher H@10, and 0.08 higher MRR than the baseline. Moreover, on the DWY100k datasets, RPE-GCN gains about 5\% higher H@1, 2\% higher H@10, and 0.05 higher MRR than the baseline. The results demonstrate the effectiveness of our proposed relational prototype entities.
 

Also, it is worth mentioning that the proposed RPE-GCN shows larger superiority on the DBP15k datasets against the recent AliNet \cite{AliNet}. On the DBP15k datasets, RPE-GCN can achieve a gain of about 5\% by H@1 and H@10, and 0.05 by MRR. However, on the DWY100k datasets, we find that RPE-GCN outperforms AliNet with an absolute margin of 0.01 by all the three metrics on DBP-YG, and obtains comparable performance to AliNet on DBP-WD. We know that the DWY100k datasets are extracted from different KGs, and the schema heterogeneity of them is much heavier than that of the DBP15k datasets
which are extracted from the multi-lingual DBpedia \cite{AliNet}. We argue that the results may be caused by that AliNet is sophisticatedly designed for mitigating the non-isomorphism of neighborhood structures for entity alignment. Therefore, AliNet is particularly suitable for the DWY100k datasets. Recall that, our proposed RPE-GCN aims at exploiting global semantic similarities in the KG, and we leave the problem of the non-isomorphism of neighborhood structures in the KG for future work.

\begin{figure*}[ht]
    \centering{
                \subfigure[The embedding visualization of RotatE (baseline)]{ \label{visualization: rotate}
            \includegraphics[height=0.8\columnwidth]{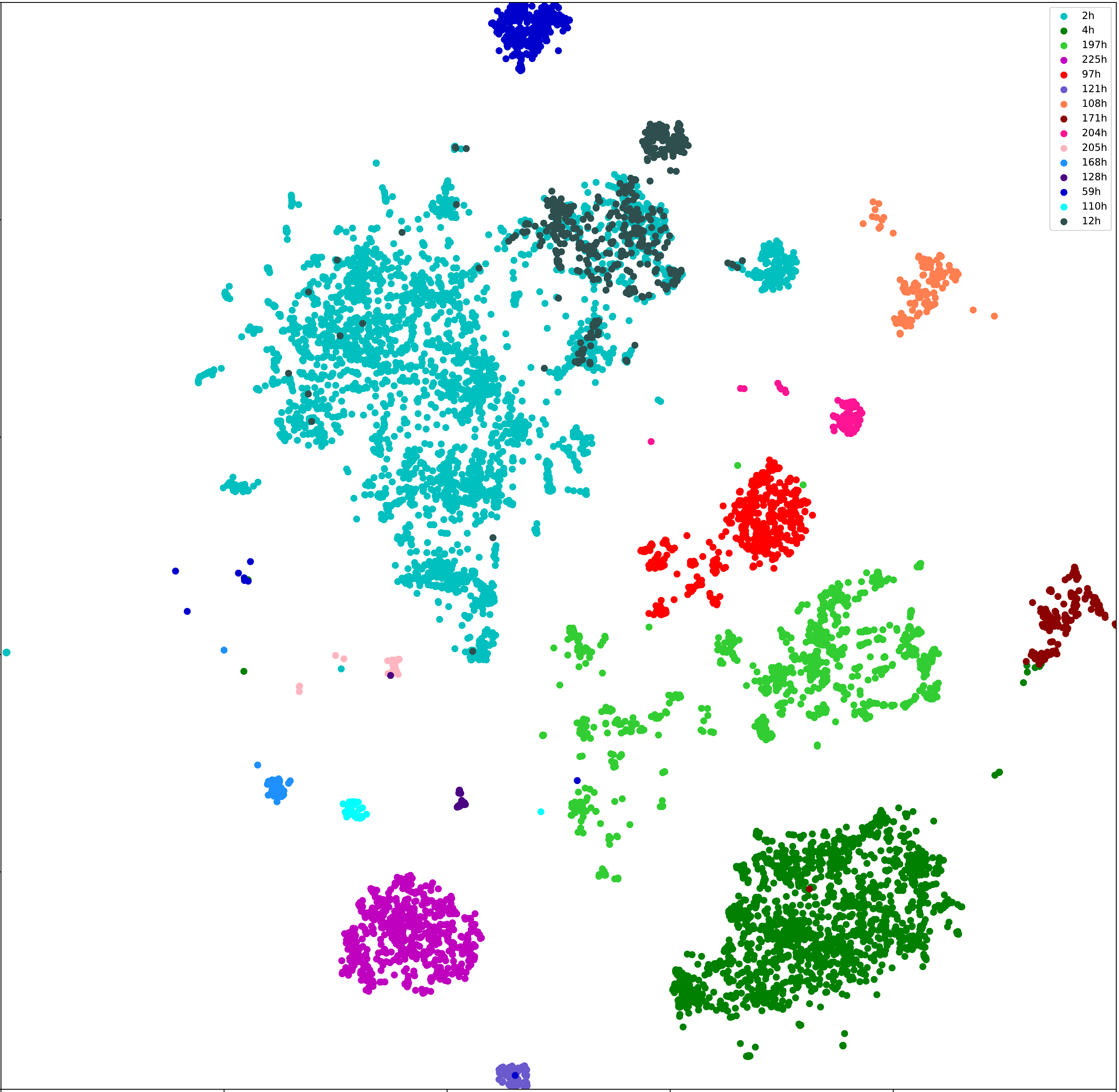}
        }\hspace{14mm}
        \subfigure[The embedding visualization of RPE-RotatE]{ \label{visualization: RPE-RotatE}
            \includegraphics[height=0.8\columnwidth]{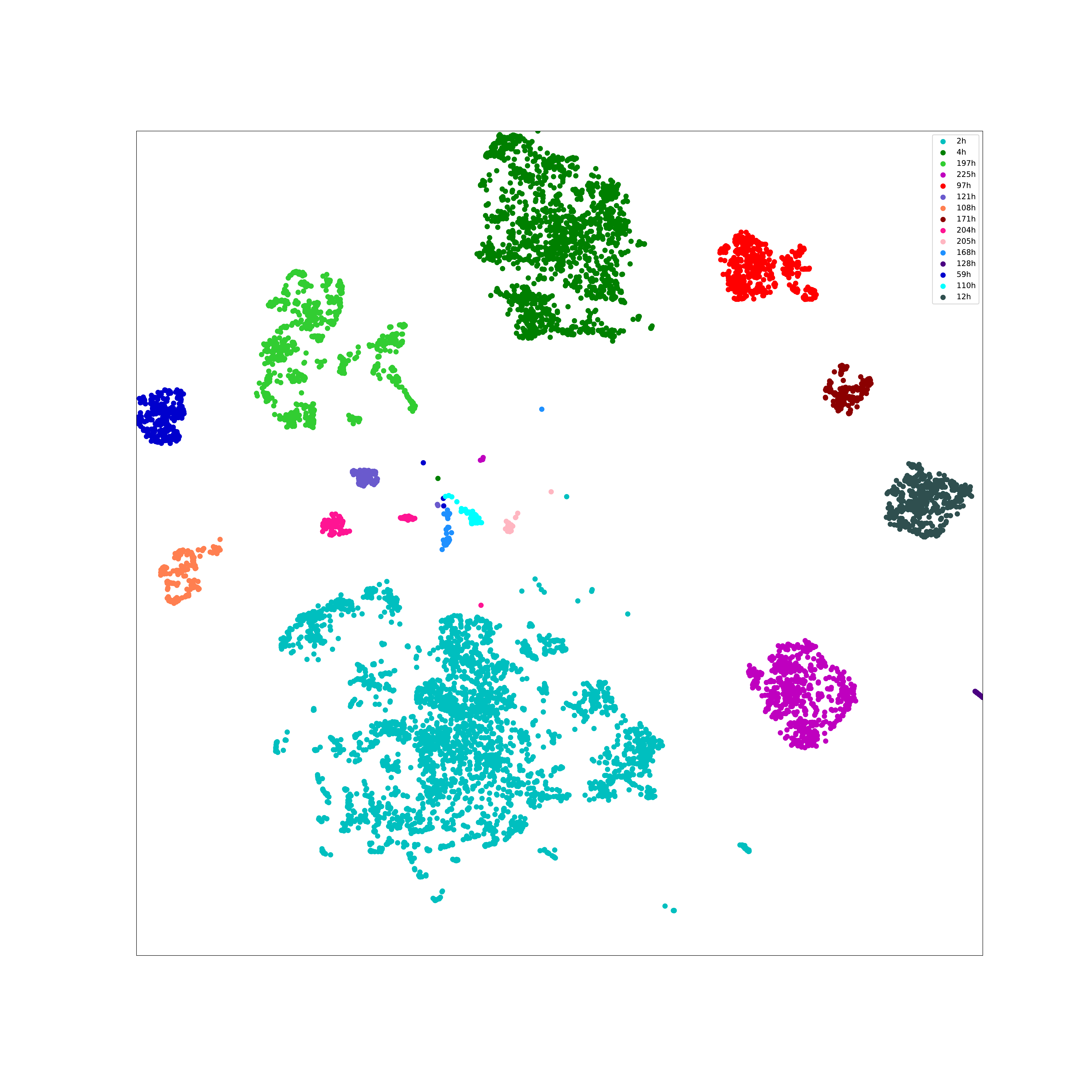}
        }
    }
    \caption{The embedding visualization of entities on the FB15k-237 dataset. Figure~\ref{visualization: rotate} and Figure~\ref{visualization: RPE-RotatE} show the entity embedding visualization of baseline and our method, respectively. Different colors represent different categories.}
    \label{visualization:kg completion}
\end{figure*}

\begin{table*}[ht]
  \caption{The Davies-Bouldin Index (DBI) performance of the baselines and our methods.}
  \label{DBI}
  \begin{tabular}{lcc|cc|cc|cc|cc|c c c}
    \toprule
    &\multicolumn{6}{c}{\textbf{${\rm DBP15k}$}}&  \multicolumn{4}{c}{\textbf{${\rm DWY100k}$}}& WN18RR&FB15k-237&YAGO3-10\\
         \cmidrule(lr){2-7}
         \cmidrule(lr){8-11}
    &ZH&EN&JA&EN&FR&EN&DBP&WD&DBP&YG\\
    \midrule
    Baseline&10.70 &12.51 &13.41 &15.18 &17.52 &15.75  &33.46   &39.53 &25.43 &65.98 &15.28 &38.86 &30.23 \\
    Our method&\textbf{6.83} &\textbf{7.86} &\textbf{7.84} &\textbf{8.78} &\textbf{10.52} &\textbf{9.86} &\textbf{9.54} &\textbf{8.20}  &\textbf{9.22} &\textbf{16.18} & \textbf{4.49} & \textbf{3.06} &\textbf{6.22} \\
  \bottomrule
\end{tabular}
\end{table*}

\begin{table}[ht]
  \caption{Statistics of the number of relations, entities, and triples in each split for WN18RR, FB15k-237 and YAGO3-10.}
  \resizebox{1.01\columnwidth}!{\begin{tabular}{lccccc}
    \toprule
Datastes&\#Relation&\#Entity&\#Training&\#Validation&\#Test\\
    \midrule
    WN18RR&11&40,943 &86,835 &3,034 &3,134\\
    FB15k-237&237 &14,541 &272,115 &17,535 &20,466\\
    YAGO3-10&37&123,182 &1,079,040 &5,000 &5,000\\
  \bottomrule
\end{tabular}}
 \label{datasets_link}
\end{table}

\subsection{KG Completion Results}
\subsubsection{\textbf{KG completion datasets.}} We evaluate our proposed models on the KG completion task on three commonly used datasets: WN18RR \cite{wn18rr}, FB15k-237 \cite{conve}, and YAGO3-10 \cite{yago3}. WN18RR, FB15k-237, and YAGO3-10 are subsets of WN18 \citep{transe}, FB15k \citep{transe}, and YAGO3 \citep{yago3}, respectively.
As pointed out by \cite{wn18rr,conve}, WN18 and FB15k suffer from the test set leakage problem through inverse relations. Therefore, we use WN18RR and FB15k-237 as the benchmarks following \cite{rotate,complex,conve}. The statistics of the datasets are summarized in Table \ref{datasets_link}.

\subsubsection{\textbf{Implementation Details}}
For comparing with the baseline RotatE fairly, the hyperparameters of RPE-RotatE are the same as those provided by the original paper \cite{rotate}. Due to lack of space, we list the hyperparameters in the Appendix \ref{hyperparameters}.

Following \cite{transe,rotate}, we evaluate the performance of KG completion task in the filtered setting; that is, we rank test triples against all other candidate triples not appearing in the training, validation, or test set, where candidates are generated by corrupting head or tail entities: $(h', r, t)$ or $(h, r, t')$. In the experiments, we also use MRR and H@N as the evaluation metrics.

\begin{figure*}[]
    \centering{
                \subfigure[${\rm DBP_{ZH-EN}}$]{ \label{task-adaptive}
            \includegraphics[height=0.43\columnwidth]{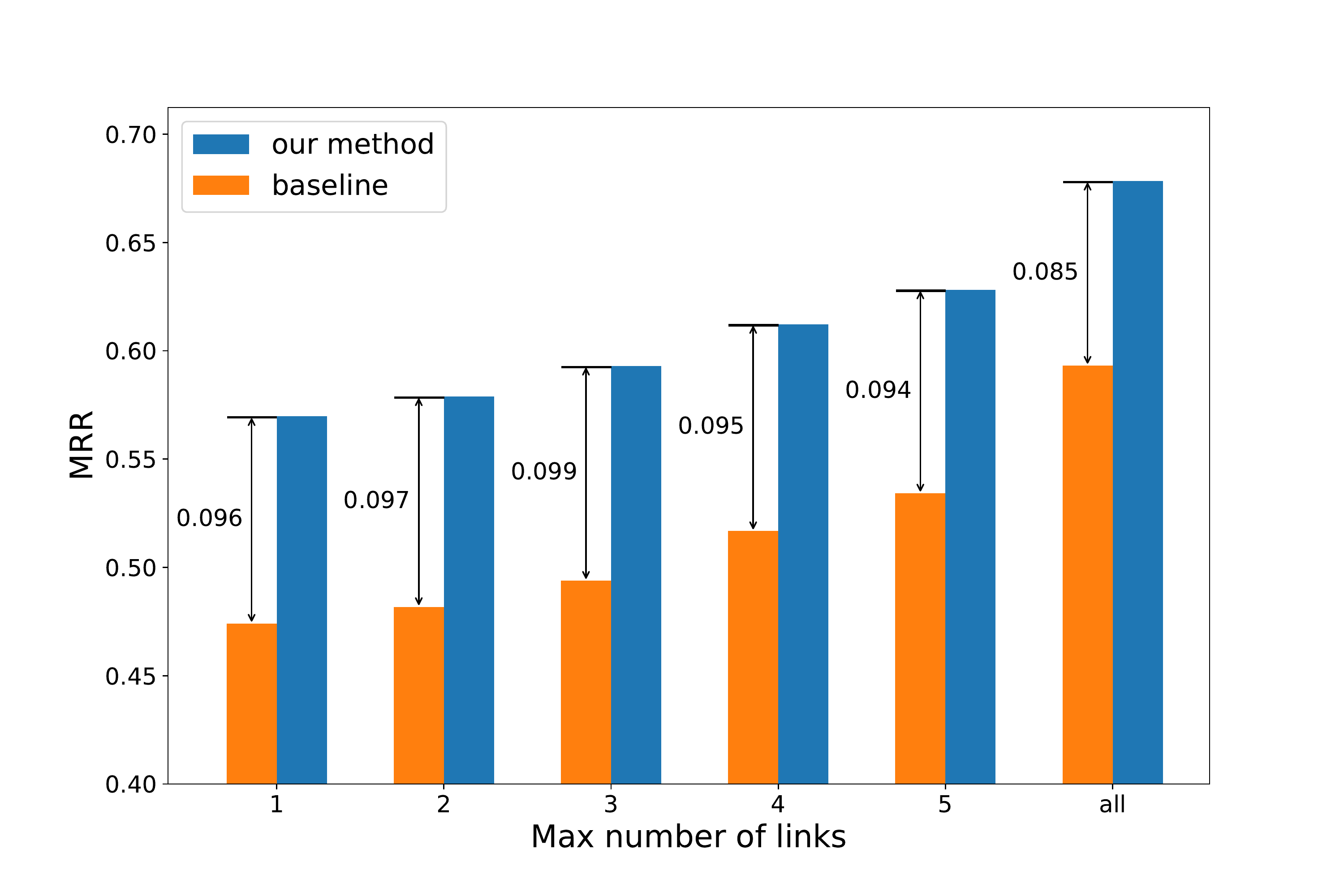}
        }\hspace{0.3mm}
        \subfigure[${\rm DBP_{JA-EN}}$]{ \label{details}
            \includegraphics[height=0.43\columnwidth]{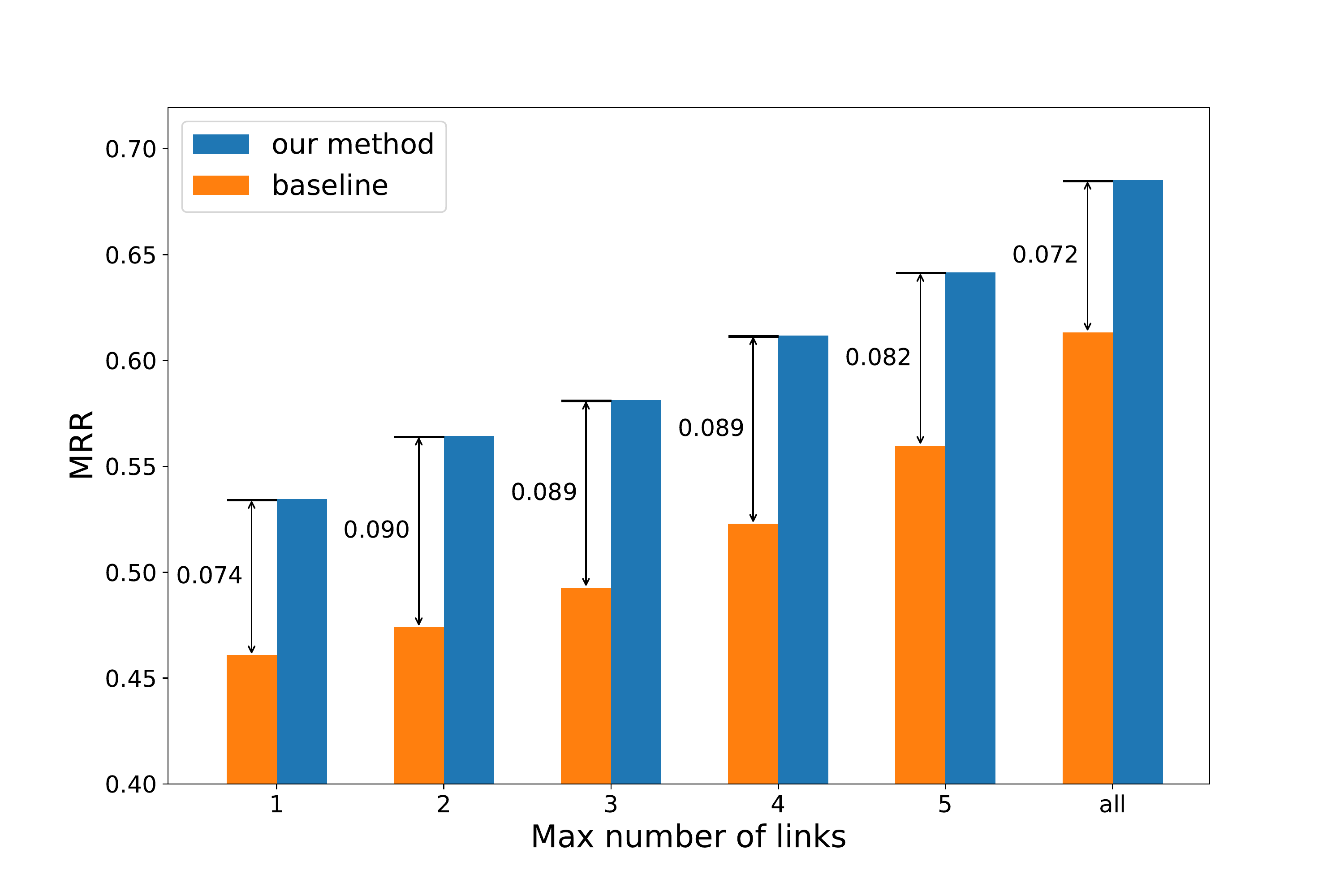}
        }\hspace{0.3mm}
        \subfigure[${\rm DBP_{FR-EN}}$]{ \label{details}
            \includegraphics[height=0.43\columnwidth]{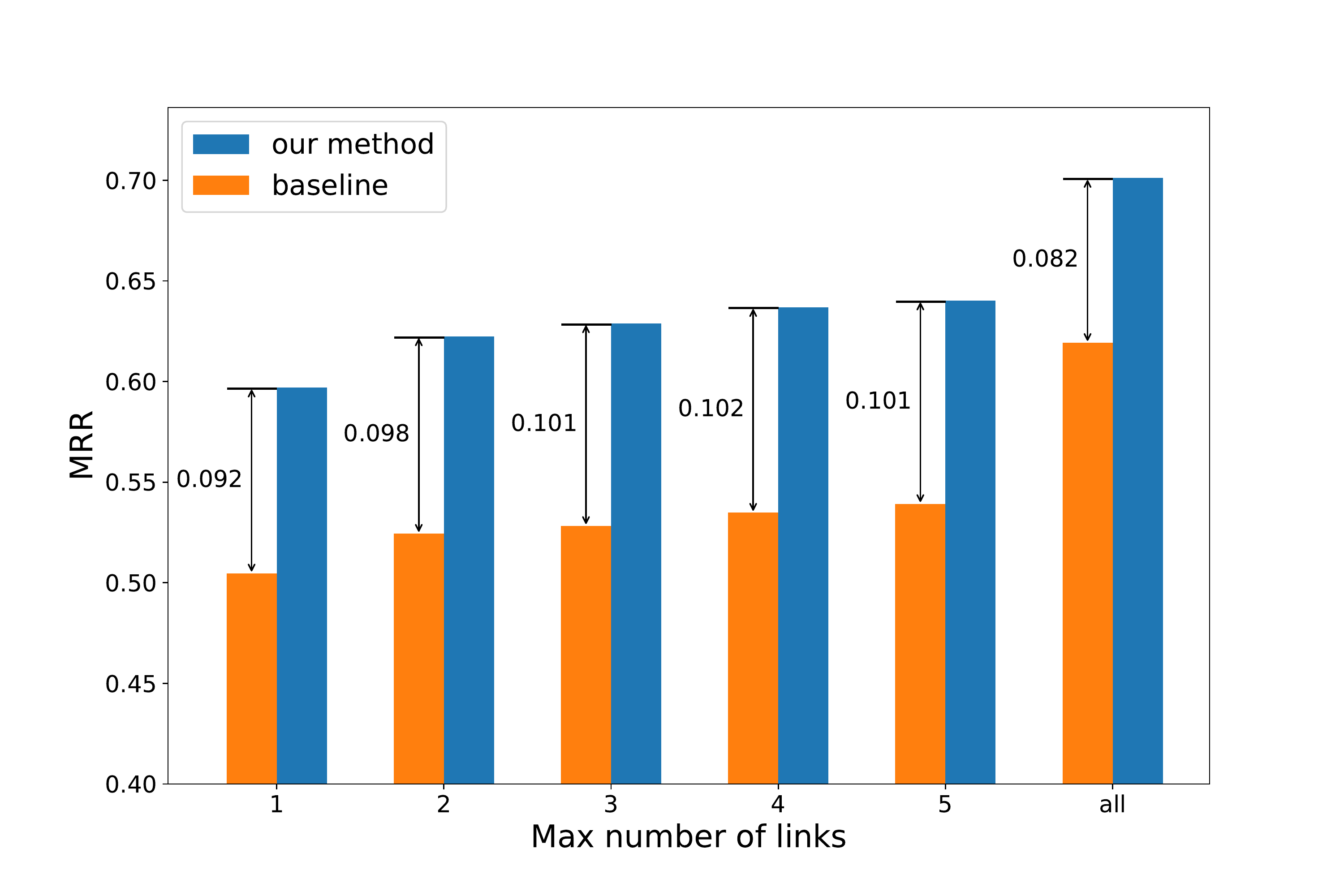}
        }
    }\\
    \centering{
                \subfigure[WN18RR]{ \label{task-adaptive}
            \includegraphics[height=0.43\columnwidth]{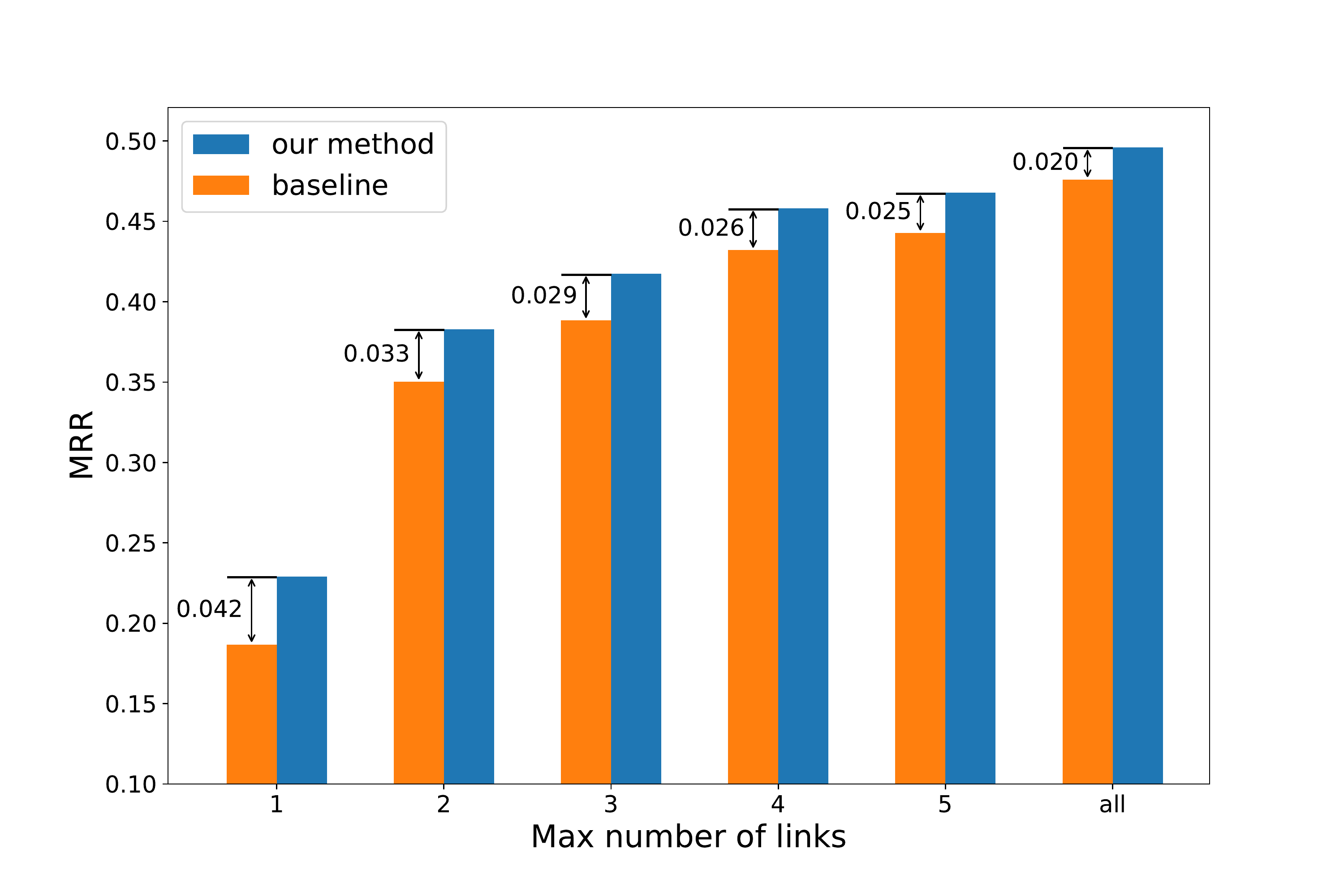}
        }
        \subfigure[FB15k-237]{ \label{details}
            \includegraphics[height=0.43\columnwidth]{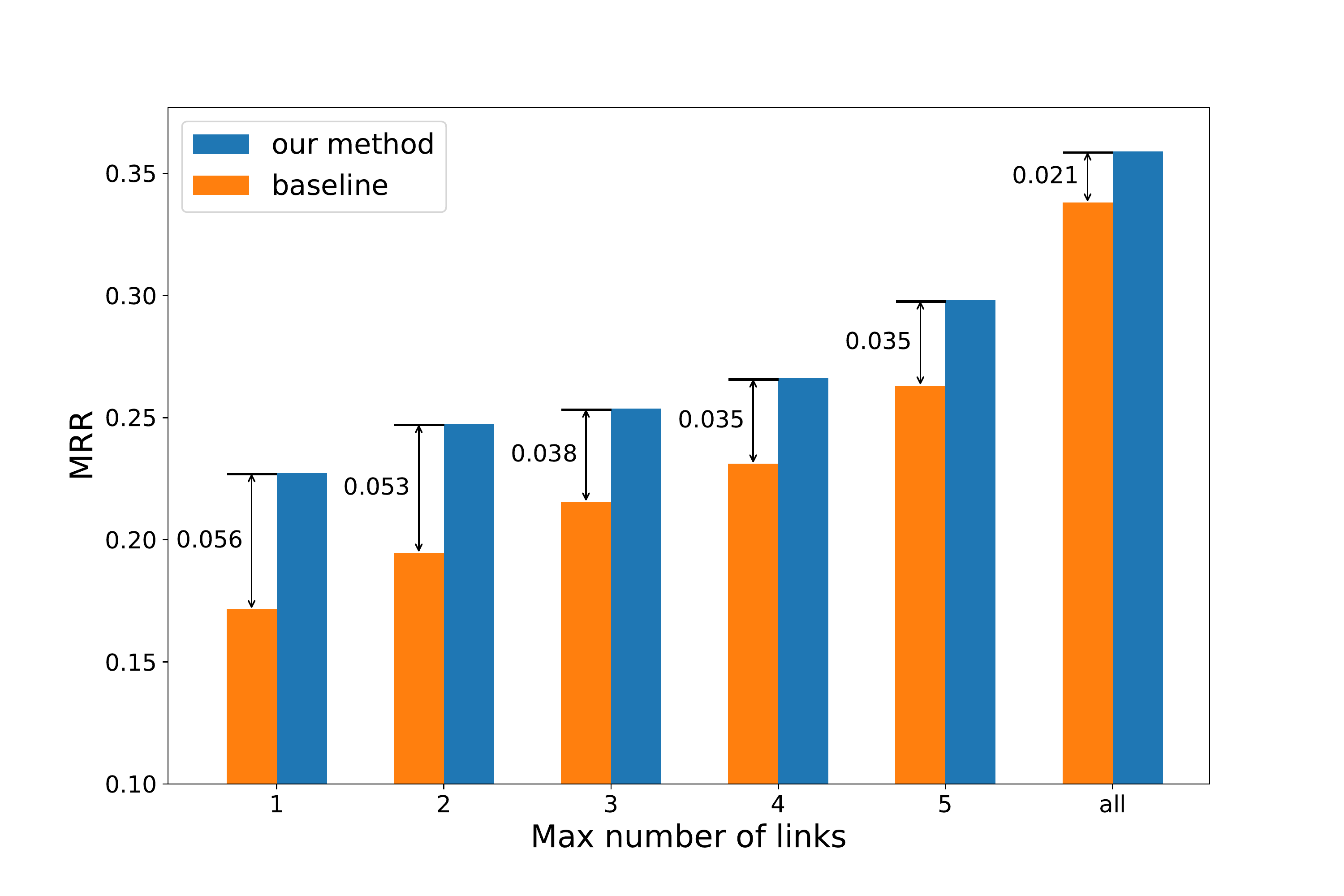}
        }
        \subfigure[YAGO3-10]{ \label{details}
            \includegraphics[height=0.43\columnwidth]{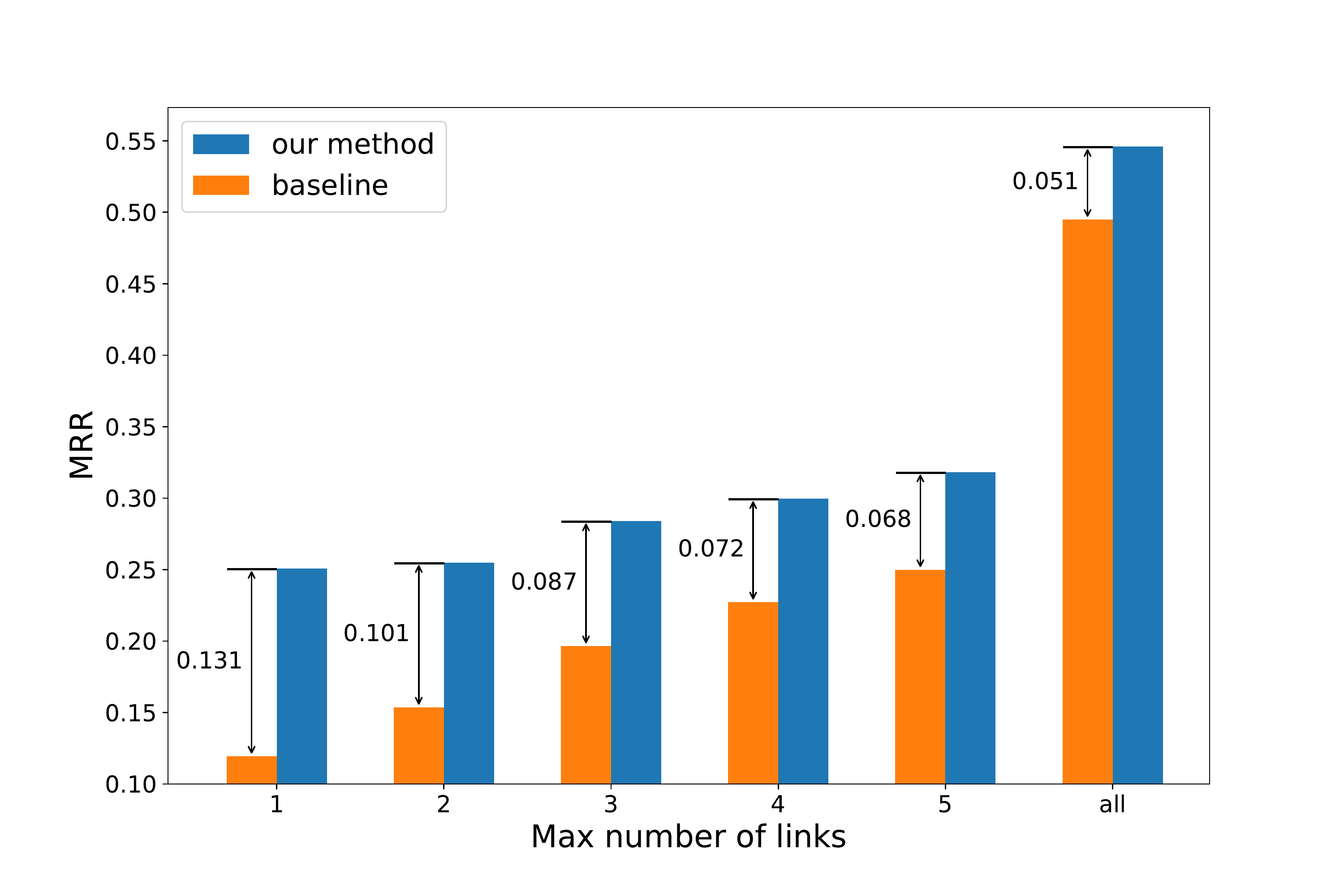}
        }
        }
    \caption{MRR results of entities w.r.t different  number of links. If the max number of links is $N$, the entities have neighbors no larger than $N$. The general trend is that the smaller the max number of links, the more significant our model will improve the performance of the baseline. }
    \label{long-tail}
\end{figure*}

\subsubsection{\textbf{Main Results}}
On the KG completion task, we compare the proposed method against several state-of-the-art methods, including TransE \cite{transe}, DistMult \cite{distmult}, ComplEx \cite{complex}, ConvE \cite{conve}, and the baseline RotatE \cite{rotate}. 

 In Table \ref{results:link prediction}, we show the performance of RPE-RotatE and several previous methods. We can see that the proposed RPE-RotatE significantly outperforms the baseline RotatE and achieves the best performance on all datasets, which demonstrates the effectiveness of our proposed method. Moreover, we can see that the improvement of RPE-RotatE on YAGO3-10 is much more significant than those on WN18RR and FB15k-237. Specifically, on the YAGO3-10 dataset, RPE-RotatE gains 0.051 higher MRR, 5.8\% higher H@1, 4.8\% higher H@3, and 3.3\% higher H@10 than RotatE. We know that the relations in the YAGO3-10 dataset have clear semantic property, e.g., \textit{directed} and \textit{hasChild}; that is, the head entities of \textit{directed} should be the ``director'', and the tail entities of it should be the ``film''. Therefore, we can expect that our proposed model is capable of working well on this dataset. WN18RR consists of relations such as \textit{hypernym} and \textit{similar to}, and such relations make the semantics of head and tail entities on this dataset less clear than YAGO3-10. FB15k-237 contains relations of composition patterns, which makes the relation types more complex than YAGO3-10.
 However, on the WN18RR and FB15k-237 datasets, our proposed RPE-RotatE can also gain better performance than RotatE, which shows that our model is widely applicable.

\section{Further Experiments} \label{futher experiments}

\subsection{The Performance of Entity Clustering}\label{performance of entity cluster}
We want to observe how our proposed method affects the embeddings of entities from the same and different categories.
Figure~\ref{visualization:kg completion} shows the embedding visualization of entities on the FB15k-237 dataset via t-SNE \cite{tsne}. Notice that, an entity may belong to more than one category, e.g., there exists $(h_1,r_1,t_1)$ and $(h_1,r_2,t_2)$; that is, $h_1$ belongs to the head categories of $r_1$ and $r_2$ at the same time. For better visualization, we choose the categories with no overlap entities. Moreover, we only show the embeddings of head categories with more than 100 entities. Figure~\ref{visualization: rotate} shows the entity embeddings of the baseline, and Figure~\ref{visualization: RPE-RotatE} shows the entity embeddings of our proposed method. Different colors represent different categories. It is readily apparent that the boundaries between different categories are more clear and the distances within categories are closer in our model, especially for the categories $110h$ and $12h$ ($12h$ represents the head category of the relation with id $12$). The visualization result demonstrates that our proposed method makes entity embeddings from different categories more distinguishable.

Moreover, we use the Davies-Bouldin Index (DBI) \cite{davies1979cluster}---a metric for evaluating clustering algorithms---to measure the performance of entity clustering of the baselines and our methods. 
DBI is a function of the ratio of the sum of within-cluster scatter to between-cluster separation. A lower DBI value means that the performance of the clustering is better. For more details about DBI, please refer to \cite{davies1979cluster}. We show the results in Table \ref{DBI}. We can see that our proposed methods achieve better performance than the baselines, which demonstrates that the proposed relational prototype entities can effectively encourage the semantic similarities of the entity embeddings from the same category.


\begin{figure}[H]
    \centering{
        \subfigure[Entity alignment task] { \label{lambda 1}
            \includegraphics[width=0.475\columnwidth]{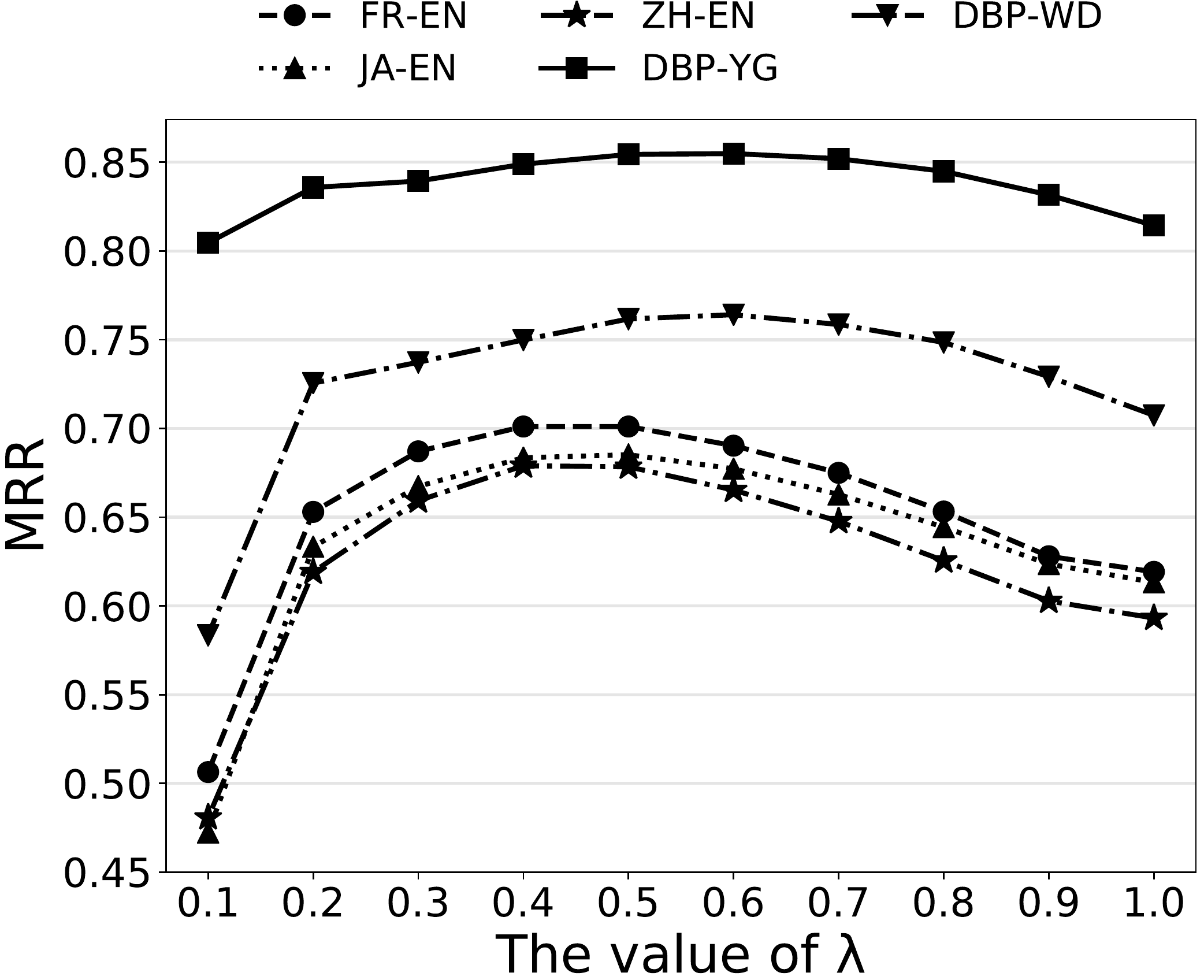}
        }
        \subfigure[KG completion task] { \label{lambda 2}
            \includegraphics[width=0.475\columnwidth]{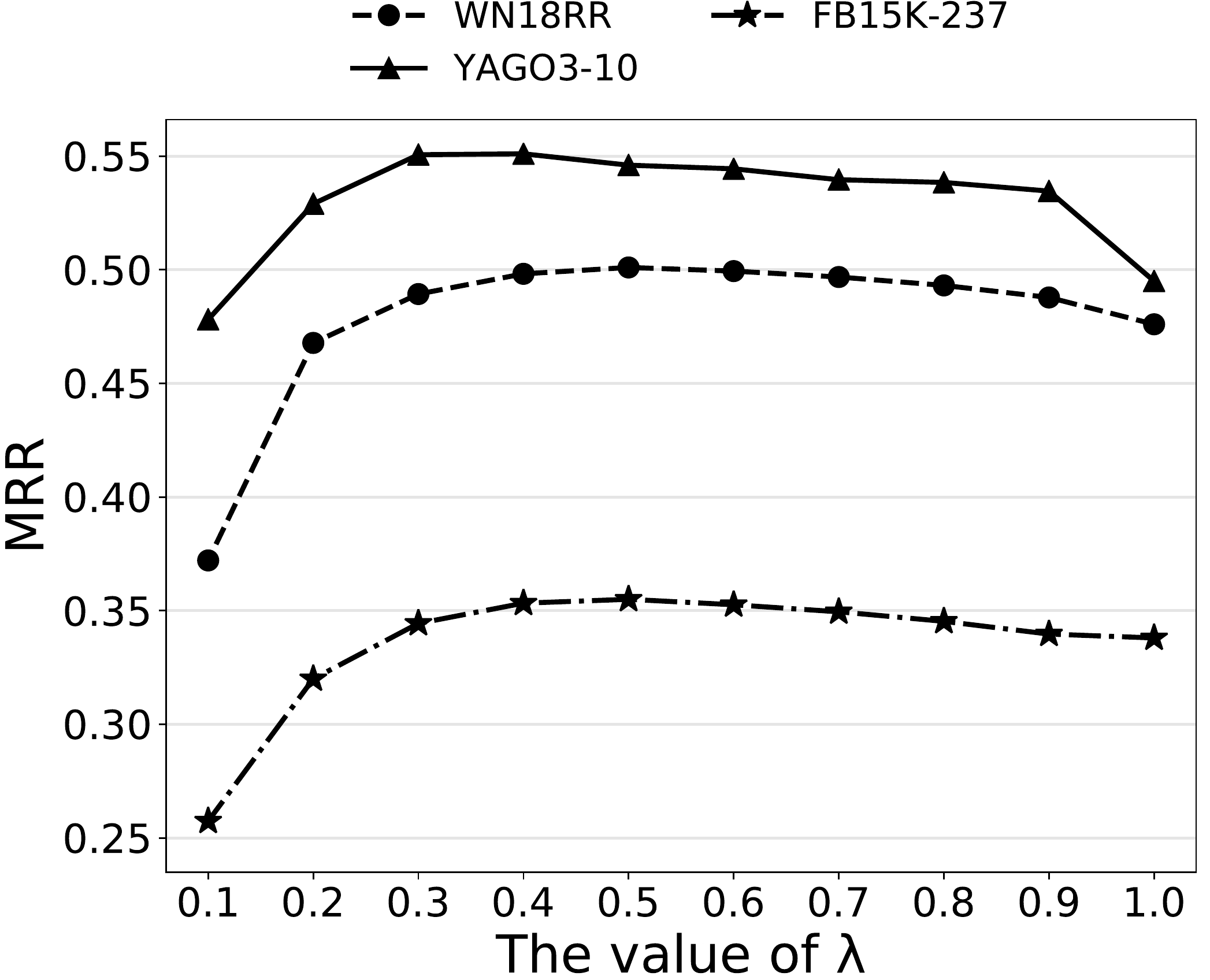}
        }
    }
    \caption{MRR results w.r.t the value of $\lambda$.}
    \label{lambda}
\end{figure}

\subsection{The Performance of Long-tail Entities}\label{performance of long-tail}
As discussed before, many existing methods focus on the local information in the KG. Under such a strategy, the long-tail entities that have few relational triples would only receive limited attention \cite{rsn}. However, our proposed method can provide rich global semantic information for these entities through relational prototype entities. Therefore, we also want to assess the effectiveness of the proposed methods on long-tail entities.

Figure~\ref{long-tail} shows the MRR results of the baselines and our methods on long-tail entities w.r.t  different number of links. Notice that, if the max number of links is $N$, it represents the entities with one-hop neighbors no larger than $N$. The \textit{`all'} represents the overall performance of the whole dataset. Due to lack of space, we only show the results of six datasets in Figure~\ref{long-tail}, and we show the results of the DBP-WD and DBP-YG datasets in the Appendix \ref{appendix c}. We can see that the general trend is that the smaller the max number of links, the more significant our model will improve the performance of the baselines. The results demonstrate the effectiveness of our methods, especially for entities with few relational triples.

\subsection{Sensitivity to The Hyperparameter $\lambda$}
In this section, we want to observe how the hyperparameter $\lambda$ affects the performance of our method. In Figure~\ref{lambda}, we show the MRR results on all the datasets w.r.t the value of $\lambda$. Recall that, when $\lambda=1$, it represents the results of the baselines (GCN and RotatE). From the results, we can see that our proposed methods outperform the baselines when $0.3\le\lambda<1$, which demonstrate the effectiveness of our approach. 
In this paper, we set $\lambda=0.5$ on the entity alignment and KG embedding tasks.

\section{Conclusion}\label{conclusion}
In this paper, we exploit global semantic similarities in the KG by relational prototype entities. By enforcing the entities’ embeddings close to their associated prototypes’ embeddings, our approach can effectively encourage the global semantic similarities of entities connected by the same relation.
Experiments on the entity alignment and KG completion tasks demonstrate that our proposed methods significantly improve the performance over the baselines, with overall performance better than recent state-of-the-art methods. 
Future work includes studying a unified architecture that can simultaneously exploit global semantic similarities and address structural heterogeneity in the knowledge graph.


\bibliographystyle{ACM-Reference-Format}
\bibliography{paper}


\begin{thebibliography}{51}


\ifx \showCODEN    \undefined \def \showCODEN     #1{\unskip}     \fi
\ifx \showDOI      \undefined \def \showDOI       #1{#1}\fi
\ifx \showISBNx    \undefined \def \showISBNx     #1{\unskip}     \fi
\ifx \showISBNxiii \undefined \def \showISBNxiii  #1{\unskip}     \fi
\ifx \showISSN     \undefined \def \showISSN      #1{\unskip}     \fi
\ifx \showLCCN     \undefined \def \showLCCN      #1{\unskip}     \fi
\ifx \shownote     \undefined \def \shownote      #1{#1}          \fi
\ifx \showarticletitle \undefined \def \showarticletitle #1{#1}   \fi
\ifx \showURL      \undefined \def \showURL       {\relax}        \fi
\providecommand\bibfield[2]{#2}
\providecommand\bibinfo[2]{#2}
\providecommand\natexlab[1]{#1}
\providecommand\showeprint[2][]{arXiv:#2}

\bibitem[Bollacker et~al\mbox{.}(2008)]%
        {freebase}
\bibfield{author}{\bibinfo{person}{Kurt Bollacker}, \bibinfo{person}{Colin
  Evans}, \bibinfo{person}{Praveen Paritosh}, \bibinfo{person}{Tim Sturge},
  {and} \bibinfo{person}{Jamie Taylor}.} \bibinfo{year}{2008}\natexlab{}.
\newblock \showarticletitle{Freebase: A Collaboratively Created Graph Database
  for Structuring Human Knowledge}. In \bibinfo{booktitle}{\emph{SIGMOD}}.
\newblock


\bibitem[Bordes et~al\mbox{.}(2013)]%
        {transe}
\bibfield{author}{\bibinfo{person}{Antoine Bordes}, \bibinfo{person}{Nicolas
  Usunier}, \bibinfo{person}{Alberto Garcia-Dur\'{a}n}, \bibinfo{person}{Jason
  Weston}, {and} \bibinfo{person}{Oksana Yakhnenko}.}
  \bibinfo{year}{2013}\natexlab{}.
\newblock \showarticletitle{Translating Embeddings for Modeling
  Multi-relational Data}. In \bibinfo{booktitle}{\emph{NeurIPS}}.
\newblock


\bibitem[Cao et~al\mbox{.}(2019)]%
        {mugnn}
\bibfield{author}{\bibinfo{person}{Yixin Cao}, \bibinfo{person}{Zhiyuan Liu},
  \bibinfo{person}{Chengjiang Li}, \bibinfo{person}{Zhiyuan Liu},
  \bibinfo{person}{Juanzi Li}, {and} \bibinfo{person}{Tat-Seng Chua}.}
  \bibinfo{year}{2019}\natexlab{}.
\newblock \showarticletitle{Multi-Channel Graph Neural Network for Entity
  Alignment}. In \bibinfo{booktitle}{\emph{ACL}}.
\newblock


\bibitem[Chen et~al\mbox{.}(2018)]%
        {chen2018co}
\bibfield{author}{\bibinfo{person}{Muhao Chen}, \bibinfo{person}{Yingtao Tian},
  \bibinfo{person}{Kai-Wei Chang}, \bibinfo{person}{Steven Skiena}, {and}
  \bibinfo{person}{Carlo Zaniolo}.} \bibinfo{year}{2018}\natexlab{}.
\newblock \showarticletitle{Co-training embeddings of knowledge graphs and
  entity descriptions for cross-lingual entity alignment}. In
  \bibinfo{booktitle}{\emph{IJCAI}}.
\newblock


\bibitem[Chen et~al\mbox{.}(2017)]%
        {mtranse}
\bibfield{author}{\bibinfo{person}{Muhao Chen}, \bibinfo{person}{Yingtao Tian},
  \bibinfo{person}{Mohan Yang}, {and} \bibinfo{person}{Carlo Zaniolo}.}
  \bibinfo{year}{2017}\natexlab{}.
\newblock \showarticletitle{Multilingual knowledge graph embeddings for
  cross-lingual knowledge alignment}. In \bibinfo{booktitle}{\emph{IJCAI}}.
\newblock


\bibitem[Davies and Bouldin(1979)]%
        {davies1979cluster}
\bibfield{author}{\bibinfo{person}{David~L Davies} {and}
  \bibinfo{person}{Donald~W Bouldin}.} \bibinfo{year}{1979}\natexlab{}.
\newblock \showarticletitle{A cluster separation measure}. In
  \bibinfo{booktitle}{\emph{TPAMI}}.
\newblock


\bibitem[Dettmers et~al\mbox{.}(2018)]%
        {conve}
\bibfield{author}{\bibinfo{person}{Tim Dettmers}, \bibinfo{person}{Minervini
  Pasquale}, \bibinfo{person}{Stenetorp Pontus}, {and}
  \bibinfo{person}{Sebastian Riedel}.} \bibinfo{year}{2018}\natexlab{}.
\newblock \showarticletitle{Convolutional 2D Knowledge Graph Embeddings}. In
  \bibinfo{booktitle}{\emph{AAAI}}.
\newblock


\bibitem[Dong et~al\mbox{.}(2014)]%
        {mlp}
\bibfield{author}{\bibinfo{person}{Xin Dong}, \bibinfo{person}{Evgeniy
  Gabrilovich}, \bibinfo{person}{Geremy Heitz}, \bibinfo{person}{Wilko Horn},
  \bibinfo{person}{Ni Lao}, \bibinfo{person}{Kevin Murphy},
  \bibinfo{person}{Thomas Strohmann}, \bibinfo{person}{Shaohua Sun}, {and}
  \bibinfo{person}{Wei Zhang}.} \bibinfo{year}{2014}\natexlab{}.
\newblock \showarticletitle{Knowledge Vault: A Web-scale Approach to
  Probabilistic Knowledge Fusion}. In \bibinfo{booktitle}{\emph{SIGKDD}}.
\newblock


\bibitem[Duchi et~al\mbox{.}([n.\,d.])]%
        {adagrad}
\bibfield{author}{\bibinfo{person}{John Duchi}, \bibinfo{person}{Elad Hazan},
  {and} \bibinfo{person}{Yoram Singer}.} \bibinfo{year}{[n.\,d.]}\natexlab{}.
\newblock \showarticletitle{Adaptive subgradient methods for online learning
  and stochastic optimization}. In \bibinfo{booktitle}{\emph{JMLR}}.
\newblock


\bibitem[Guo et~al\mbox{.}(2019)]%
        {rsn}
\bibfield{author}{\bibinfo{person}{Lingbing Guo}, \bibinfo{person}{Zequn Sun},
  {and} \bibinfo{person}{Wei Hu}.} \bibinfo{year}{2019}\natexlab{}.
\newblock \showarticletitle{Learning to Exploit Long-term Relational
  Dependencies in Knowledge Graphs}. In \bibinfo{booktitle}{\emph{ICML}}.
\newblock


\bibitem[Guo et~al\mbox{.}(2015)]%
        {guo2015semantically}
\bibfield{author}{\bibinfo{person}{Shu Guo}, \bibinfo{person}{Quan Wang},
  \bibinfo{person}{Bin Wang}, \bibinfo{person}{Lihong Wang}, {and}
  \bibinfo{person}{Li Guo}.} \bibinfo{year}{2015}\natexlab{}.
\newblock \showarticletitle{Semantically smooth knowledge graph embedding}. In
  \bibinfo{booktitle}{\emph{IJCNLP}}.
\newblock


\bibitem[Huang et~al\mbox{.}(2019)]%
        {KGQA}
\bibfield{author}{\bibinfo{person}{Xiao Huang}, \bibinfo{person}{Jingyuan
  Zhang}, \bibinfo{person}{Dingcheng Li}, {and} \bibinfo{person}{Ping Li}.}
  \bibinfo{year}{2019}\natexlab{}.
\newblock \showarticletitle{Knowledge Graph Embedding Based Question
  Answering}. In \bibinfo{booktitle}{\emph{WSDM}}.
\newblock


\bibitem[Ji et~al\mbox{.}(2015)]%
        {transd}
\bibfield{author}{\bibinfo{person}{Guoliang Ji}, \bibinfo{person}{Shizhu He},
  \bibinfo{person}{Liheng Xu}, \bibinfo{person}{Kang Liu}, {and}
  \bibinfo{person}{Jun Zhao}.} \bibinfo{year}{2015}\natexlab{}.
\newblock \showarticletitle{Knowledge Graph Embedding via Dynamic Mapping
  Matrix}. In \bibinfo{booktitle}{\emph{ACL}}.
\newblock


\bibitem[Ji et~al\mbox{.}(2016)]%
        {transparse}
\bibfield{author}{\bibinfo{person}{Guoliang Ji}, \bibinfo{person}{Kang Liu},
  \bibinfo{person}{Shizhu He}, {and} \bibinfo{person}{Jun Zhao}.}
  \bibinfo{year}{2016}\natexlab{}.
\newblock \showarticletitle{Knowledge graph completion with adaptive sparse
  transfer matrix}. In \bibinfo{booktitle}{\emph{AAAI}}.
\newblock


\bibitem[Kingma and Ba(2015)]%
        {adam}
\bibfield{author}{\bibinfo{person}{Diederick~P Kingma} {and}
  \bibinfo{person}{Jimmy Ba}.} \bibinfo{year}{2015}\natexlab{}.
\newblock \showarticletitle{Adam: A method for stochastic optimization}. In
  \bibinfo{booktitle}{\emph{ICLR}}.
\newblock


\bibitem[Kipf and Welling(2017)]%
        {gcn}
\bibfield{author}{\bibinfo{person}{Thomas~N Kipf} {and} \bibinfo{person}{Max
  Welling}.} \bibinfo{year}{2017}\natexlab{}.
\newblock \showarticletitle{Semi-supervised classification with graph
  convolutional networks}. In \bibinfo{booktitle}{\emph{ICLR}}.
\newblock


\bibitem[Krompa{\ss} et~al\mbox{.}(2015)]%
        {krompass2015type}
\bibfield{author}{\bibinfo{person}{Denis Krompa{\ss}}, \bibinfo{person}{Stephan
  Baier}, {and} \bibinfo{person}{Volker Tresp}.}
  \bibinfo{year}{2015}\natexlab{}.
\newblock \showarticletitle{Type-constrained representation learning in
  knowledge graphs}. In \bibinfo{booktitle}{\emph{ISWC}}.
\newblock


\bibitem[Lehmann et~al\mbox{.}(2015)]%
        {dbpedia}
\bibfield{author}{\bibinfo{person}{Jens Lehmann}, \bibinfo{person}{Robert
  Isele}, \bibinfo{person}{Max Jakob}, \bibinfo{person}{Anja Jentzsch},
  \bibinfo{person}{Dimitris Kontokostas}, \bibinfo{person}{Pablo~N Mendes},
  \bibinfo{person}{Sebastian Hellmann}, \bibinfo{person}{Mohamed Morsey},
  \bibinfo{person}{Patrick Van~Kleef}, \bibinfo{person}{S{\"o}ren Auer},
  {et~al\mbox{.}}} \bibinfo{year}{2015}\natexlab{}.
\newblock \showarticletitle{DBpedia--a large-scale, multilingual knowledge base
  extracted from Wikipedia}. In \bibinfo{booktitle}{\emph{Semantic Web}}.
\newblock


\bibitem[Lin et~al\mbox{.}(2015a)]%
        {ptranse}
\bibfield{author}{\bibinfo{person}{Yankai Lin}, \bibinfo{person}{Zhiyuan Liu},
  \bibinfo{person}{Huanbo Luan}, \bibinfo{person}{Maosong Sun},
  \bibinfo{person}{Siwei Rao}, {and} \bibinfo{person}{Song Liu}.}
  \bibinfo{year}{2015}\natexlab{a}.
\newblock \showarticletitle{Modeling relation paths for representation learning
  of knowledge bases}. In \bibinfo{booktitle}{\emph{EMNLP}}.
\newblock


\bibitem[Lin et~al\mbox{.}(2015b)]%
        {transr}
\bibfield{author}{\bibinfo{person}{Yankai Lin}, \bibinfo{person}{Zhiyuan Liu},
  \bibinfo{person}{Maosong Sun}, \bibinfo{person}{Yang Liu}, {and}
  \bibinfo{person}{Xuan Zhu}.} \bibinfo{year}{2015}\natexlab{b}.
\newblock \showarticletitle{Learning Entity and Relation Embeddings for
  Knowledge Graph Completion}. In \bibinfo{booktitle}{\emph{AAAI}}.
\newblock


\bibitem[Liu et~al\mbox{.}(2017)]%
        {analogy}
\bibfield{author}{\bibinfo{person}{Hanxiao Liu}, \bibinfo{person}{Yuexin Wu},
  {and} \bibinfo{person}{Yiming Yang}.} \bibinfo{year}{2017}\natexlab{}.
\newblock \showarticletitle{Analogical Inference for Multi-relational
  Embeddings}. In \bibinfo{booktitle}{\emph{ICML}}.
\newblock


\bibitem[Ma et~al\mbox{.}(2017)]%
        {transt}
\bibfield{author}{\bibinfo{person}{Shiheng Ma}, \bibinfo{person}{Jianhui Ding},
  \bibinfo{person}{Weijia Jia}, \bibinfo{person}{Kun Wang}, {and}
  \bibinfo{person}{Minyi Guo}.} \bibinfo{year}{2017}\natexlab{}.
\newblock \showarticletitle{Transt: Type-based multiple embedding
  representations for knowledge graph completion}. In
  \bibinfo{booktitle}{\emph{ECML PKDD}}.
\newblock


\bibitem[Maaten and Hinton(2008)]%
        {tsne}
\bibfield{author}{\bibinfo{person}{Laurens van~der Maaten} {and}
  \bibinfo{person}{Geoffrey Hinton}.} \bibinfo{year}{2008}\natexlab{}.
\newblock \showarticletitle{Visualizing data using t-SNE}. In
  \bibinfo{booktitle}{\emph{JMLR}}.
\newblock


\bibitem[Mahdisoltani et~al\mbox{.}(2013)]%
        {yago3}
\bibfield{author}{\bibinfo{person}{Farzaneh Mahdisoltani},
  \bibinfo{person}{Joanna Biega}, {and} \bibinfo{person}{Fabian~M Suchanek}.}
  \bibinfo{year}{2013}\natexlab{}.
\newblock \showarticletitle{Yago3: A knowledge base from multilingual
  wikipedias}. In \bibinfo{booktitle}{\emph{CIDR}}.
\newblock


\bibitem[Mikolov et~al\mbox{.}(2013)]%
        {mikolov2013distributed}
\bibfield{author}{\bibinfo{person}{Tomas Mikolov}, \bibinfo{person}{Ilya
  Sutskever}, \bibinfo{person}{Kai Chen}, \bibinfo{person}{Greg~S Corrado},
  {and} \bibinfo{person}{Jeff Dean}.} \bibinfo{year}{2013}\natexlab{}.
\newblock \showarticletitle{Distributed representations of words and phrases
  and their compositionality}. In \bibinfo{booktitle}{\emph{NeurIPS}}.
\newblock


\bibitem[Nguyen et~al\mbox{.}(2018)]%
        {convkb}
\bibfield{author}{\bibinfo{person}{Dai~Quoc Nguyen}, \bibinfo{person}{Tu~Dinh
  Nguyen}, \bibinfo{person}{Dat~Quoc Nguyen}, {and} \bibinfo{person}{Dinh
  Phung}.} \bibinfo{year}{2018}\natexlab{}.
\newblock \showarticletitle{A Novel Embedding Model for Knowledge Base
  Completion Based on Convolutional Neural Network}. In
  \bibinfo{booktitle}{\emph{NAACL}}.
\newblock


\bibitem[Nickel et~al\mbox{.}(2016)]%
        {hole}
\bibfield{author}{\bibinfo{person}{Maximilian Nickel}, \bibinfo{person}{Lorenzo
  Rosasco}, {and} \bibinfo{person}{Tomaso Poggio}.}
  \bibinfo{year}{2016}\natexlab{}.
\newblock \showarticletitle{Holographic Embeddings of Knowledge Graphs}. In
  \bibinfo{booktitle}{\emph{AAAI}}.
\newblock


\bibitem[Nickel et~al\mbox{.}(2011)]%
        {rescal}
\bibfield{author}{\bibinfo{person}{Maximilian Nickel}, \bibinfo{person}{Volker
  Tresp}, {and} \bibinfo{person}{Hans-Peter Kriegel}.}
  \bibinfo{year}{2011}\natexlab{}.
\newblock \showarticletitle{A Three-way Model for Collective Learning on
  Multi-relational Data}. In \bibinfo{booktitle}{\emph{ICML}}.
\newblock


\bibitem[Pei et~al\mbox{.}(2019)]%
        {sea}
\bibfield{author}{\bibinfo{person}{Shichao Pei}, \bibinfo{person}{Lu Yu},
  \bibinfo{person}{Robert Hoehndorf}, {and} \bibinfo{person}{Xiangliang
  Zhang}.} \bibinfo{year}{2019}\natexlab{}.
\newblock \showarticletitle{Semi-Supervised Entity Alignment via Knowledge
  Graph Embedding with Awareness of Degree Difference}. In
  \bibinfo{booktitle}{\emph{WWW}}.
\newblock


\bibitem[Pujara et~al\mbox{.}(2013)]%
        {pujara2013knowledge}
\bibfield{author}{\bibinfo{person}{Jay Pujara}, \bibinfo{person}{Hui Miao},
  \bibinfo{person}{Lise Getoor}, {and} \bibinfo{person}{William Cohen}.}
  \bibinfo{year}{2013}\natexlab{}.
\newblock \showarticletitle{Knowledge graph identification}. In
  \bibinfo{booktitle}{\emph{ISWC}}.
\newblock


\bibitem[Schlichtkrull et~al\mbox{.}(2018)]%
        {rgcn}
\bibfield{author}{\bibinfo{person}{Michael Schlichtkrull},
  \bibinfo{person}{Thomas~N Kipf}, \bibinfo{person}{Peter Bloem},
  \bibinfo{person}{Rianne Van Den~Berg}, \bibinfo{person}{Ivan Titov}, {and}
  \bibinfo{person}{Max Welling}.} \bibinfo{year}{2018}\natexlab{}.
\newblock \showarticletitle{Modeling relational data with graph convolutional
  networks}. In \bibinfo{booktitle}{\emph{ESWC}}.
\newblock


\bibitem[Socher et~al\mbox{.}(2013)]%
        {ntn}
\bibfield{author}{\bibinfo{person}{Richard Socher}, \bibinfo{person}{Danqi
  Chen}, \bibinfo{person}{Christopher~D. Manning}, {and}
  \bibinfo{person}{Andrew~Y. Ng}.} \bibinfo{year}{2013}\natexlab{}.
\newblock \showarticletitle{Reasoning with Neural Tensor Networks for Knowledge
  Base Completion}. In \bibinfo{booktitle}{\emph{NeurIPS}}.
\newblock


\bibitem[Suchanek et~al\mbox{.}(2007)]%
        {yago}
\bibfield{author}{\bibinfo{person}{Fabian~M Suchanek}, \bibinfo{person}{Gjergji
  Kasneci}, {and} \bibinfo{person}{Gerhard Weikum}.}
  \bibinfo{year}{2007}\natexlab{}.
\newblock \showarticletitle{Yago: a core of semantic knowledge}. In
  \bibinfo{booktitle}{\emph{WWW}}.
\newblock


\bibitem[Sun et~al\mbox{.}(2019)]%
        {rotate}
\bibfield{author}{\bibinfo{person}{Zhiqing Sun}, \bibinfo{person}{Zhi-Hong
  Deng}, \bibinfo{person}{Jian-Yun Nie}, {and} \bibinfo{person}{Jian Tang}.}
  \bibinfo{year}{2019}\natexlab{}.
\newblock \showarticletitle{RotatE: Knowledge Graph Embedding by Relational
  Rotation in Complex Space}. In \bibinfo{booktitle}{\emph{ICLR}}.
\newblock


\bibitem[Sun et~al\mbox{.}(2017)]%
        {sun2017cross}
\bibfield{author}{\bibinfo{person}{Zequn Sun}, \bibinfo{person}{Wei Hu}, {and}
  \bibinfo{person}{Chengkai Li}.} \bibinfo{year}{2017}\natexlab{}.
\newblock \showarticletitle{Cross-lingual entity alignment via joint
  attribute-preserving embedding}. In \bibinfo{booktitle}{\emph{ISWC}}.
\newblock


\bibitem[Sun et~al\mbox{.}(2018)]%
        {sun2018bootstrapping}
\bibfield{author}{\bibinfo{person}{Zequn Sun}, \bibinfo{person}{Wei Hu},
  \bibinfo{person}{Qingheng Zhang}, {and} \bibinfo{person}{Yuzhong Qu}.}
  \bibinfo{year}{2018}\natexlab{}.
\newblock \showarticletitle{Bootstrapping Entity Alignment with Knowledge Graph
  Embedding.}. In \bibinfo{booktitle}{\emph{IJCAI}}.
\newblock


\bibitem[Sun et~al\mbox{.}(2020)]%
        {AliNet}
\bibfield{author}{\bibinfo{person}{Zequn Sun}, \bibinfo{person}{Chengming
  Wang}, \bibinfo{person}{Wei Hu}, \bibinfo{person}{Muhao Chen},
  \bibinfo{person}{Jian Dai}, \bibinfo{person}{Wei Zhang}, {and}
  \bibinfo{person}{Yuzhong Qu}.} \bibinfo{year}{2020}\natexlab{}.
\newblock \showarticletitle{Knowledge Graph Alignment Network with Gated
  Multi-hop Neighborhood Aggregation}. In \bibinfo{booktitle}{\emph{AAAI}}.
\newblock


\bibitem[Toutanova and Chen(2015)]%
        {wn18rr}
\bibfield{author}{\bibinfo{person}{Kristina Toutanova} {and}
  \bibinfo{person}{Danqi Chen}.} \bibinfo{year}{2015}\natexlab{}.
\newblock \showarticletitle{Observed versus latent features for knowledge base
  and text inference}. In \bibinfo{booktitle}{\emph{The Workshop on CVSC}}.
\newblock


\bibitem[Trouillon et~al\mbox{.}(2016)]%
        {complex}
\bibfield{author}{\bibinfo{person}{Th{\'e}o Trouillon},
  \bibinfo{person}{Johannes Welbl}, \bibinfo{person}{Sebastian Riedel},
  \bibinfo{person}{\'{E}ric Gaussier}, {and} \bibinfo{person}{Guillaume
  Bouchard}.} \bibinfo{year}{2016}\natexlab{}.
\newblock \showarticletitle{Complex Embeddings for Simple Link Prediction}. In
  \bibinfo{booktitle}{\emph{ICML}}.
\newblock


\bibitem[Wang et~al\mbox{.}(2018b)]%
        {KGRS}
\bibfield{author}{\bibinfo{person}{Hongwei Wang}, \bibinfo{person}{Fuzheng
  Zhang}, \bibinfo{person}{Jialin Wang}, \bibinfo{person}{Miao Zhao},
  \bibinfo{person}{Wenjie Li}, \bibinfo{person}{Xing Xie}, {and}
  \bibinfo{person}{Minyi Guo}.} \bibinfo{year}{2018}\natexlab{b}.
\newblock \showarticletitle{RippleNet: Propagating User Preferences on the
  Knowledge Graph for Recommender Systems}. In
  \bibinfo{booktitle}{\emph{CIKM}}.
\newblock


\bibitem[Wang et~al\mbox{.}(2017)]%
        {KGEsurvey}
\bibfield{author}{\bibinfo{person}{Quan Wang}, \bibinfo{person}{Zhendong Mao},
  \bibinfo{person}{Bin Wang}, {and} \bibinfo{person}{Li Guo}.}
  \bibinfo{year}{2017}\natexlab{}.
\newblock \showarticletitle{Knowledge Graph Embedding: A Survey of Approaches
  and Applications}. In \bibinfo{booktitle}{\emph{TKDE}}.
\newblock


\bibitem[Wang et~al\mbox{.}(2019)]%
        {kgat}
\bibfield{author}{\bibinfo{person}{Xiang Wang}, \bibinfo{person}{Xiangnan He},
  \bibinfo{person}{Yixin Cao}, \bibinfo{person}{Meng Liu}, {and}
  \bibinfo{person}{Tat-Seng Chua}.} \bibinfo{year}{2019}\natexlab{}.
\newblock \showarticletitle{KGAT: Knowledge Graph Attention Network for
  Recommendation}. In \bibinfo{booktitle}{\emph{SIGKDD}}.
\newblock


\bibitem[Wang et~al\mbox{.}(2018a)]%
        {gcn-align}
\bibfield{author}{\bibinfo{person}{Zhichun Wang}, \bibinfo{person}{Qingsong
  Lv}, \bibinfo{person}{Xiaohan Lan}, {and} \bibinfo{person}{Yu Zhang}.}
  \bibinfo{year}{2018}\natexlab{a}.
\newblock \showarticletitle{Cross-lingual knowledge graph alignment via graph
  convolutional networks}. In \bibinfo{booktitle}{\emph{EMNLP}}.
\newblock


\bibitem[Wang et~al\mbox{.}(2014)]%
        {transh}
\bibfield{author}{\bibinfo{person}{Zhen Wang}, \bibinfo{person}{Jianwen Zhang},
  \bibinfo{person}{Jianlin Feng}, {and} \bibinfo{person}{Zheng Chen}.}
  \bibinfo{year}{2014}\natexlab{}.
\newblock \showarticletitle{Knowledge Graph Embedding by Translating on
  Hyperplanes}. In \bibinfo{booktitle}{\emph{AAAI}}.
\newblock


\bibitem[Wu et~al\mbox{.}(2019)]%
        {rdgcn}
\bibfield{author}{\bibinfo{person}{Yuting Wu}, \bibinfo{person}{Xiao Liu},
  \bibinfo{person}{Yansong Feng}, \bibinfo{person}{Zheng Wang},
  \bibinfo{person}{Rui Yan}, {and} \bibinfo{person}{Dongyan Zhao}.}
  \bibinfo{year}{2019}\natexlab{}.
\newblock \showarticletitle{Relation-aware entity alignment for heterogeneous
  knowledge graphs}. In \bibinfo{booktitle}{\emph{IJCAI}}.
\newblock


\bibitem[Xie et~al\mbox{.}(2016)]%
        {xie2016representation}
\bibfield{author}{\bibinfo{person}{Ruobing Xie}, \bibinfo{person}{Zhiyuan Liu},
  {and} \bibinfo{person}{Maosong Sun}.} \bibinfo{year}{2016}\natexlab{}.
\newblock \showarticletitle{Representation Learning of Knowledge Graphs with
  Hierarchical Types.}. In \bibinfo{booktitle}{\emph{IJCAI}}.
\newblock


\bibitem[Xiong et~al\mbox{.}(2017)]%
        {xiong2017explicit}
\bibfield{author}{\bibinfo{person}{Chenyan Xiong}, \bibinfo{person}{Russell
  Power}, {and} \bibinfo{person}{Jamie Callan}.}
  \bibinfo{year}{2017}\natexlab{}.
\newblock \showarticletitle{Explicit semantic ranking for academic search via
  knowledge graph embedding}. In \bibinfo{booktitle}{\emph{WWW}}.
\newblock


\bibitem[Xu et~al\mbox{.}(2019)]%
        {gmnn}
\bibfield{author}{\bibinfo{person}{Kun Xu}, \bibinfo{person}{Liwei Wang},
  \bibinfo{person}{Mo Yu}, \bibinfo{person}{Yansong Feng}, \bibinfo{person}{Yan
  Song}, \bibinfo{person}{Zhiguo Wang}, {and} \bibinfo{person}{Dong Yu}.}
  \bibinfo{year}{2019}\natexlab{}.
\newblock \showarticletitle{Cross-lingual Knowledge Graph Alignment via Graph
  Matching Neural Network}. In \bibinfo{booktitle}{\emph{ACL}}.
\newblock


\bibitem[Yang et~al\mbox{.}(2015)]%
        {distmult}
\bibfield{author}{\bibinfo{person}{Bishan Yang}, \bibinfo{person}{Scott Wen-tau
  Yih}, \bibinfo{person}{Xiaodong He}, \bibinfo{person}{Jianfeng Gao}, {and}
  \bibinfo{person}{Li Deng}.} \bibinfo{year}{2015}\natexlab{}.
\newblock \showarticletitle{Embedding Entities and Relations for Learning and
  Inference in Knowledge Bases}. In \bibinfo{booktitle}{\emph{ICLR}}.
\newblock


\bibitem[Ye et~al\mbox{.}(2019)]%
        {avrgcn}
\bibfield{author}{\bibinfo{person}{Rui Ye}, \bibinfo{person}{Xin Li},
  \bibinfo{person}{Yujie Fang}, \bibinfo{person}{Hongyu Zang}, {and}
  \bibinfo{person}{Mingzhong Wang}.} \bibinfo{year}{2019}\natexlab{}.
\newblock \showarticletitle{A vectorized relational graph convolutional network
  for multi-relational network alignment}. In
  \bibinfo{booktitle}{\emph{IJCAI}}.
\newblock


\bibitem[Zhu et~al\mbox{.}(2017)]%
        {iptranse}
\bibfield{author}{\bibinfo{person}{Hao Zhu}, \bibinfo{person}{Ruobing Xie},
  \bibinfo{person}{Zhiyuan Liu}, {and} \bibinfo{person}{Maosong Sun}.}
  \bibinfo{year}{2017}\natexlab{}.
\newblock \showarticletitle{Iterative Entity Alignment via Joint Knowledge
  Embeddings.}. In \bibinfo{booktitle}{\emph{IJCAI}}.
\newblock


\end{thebibliography}

\newpage
\appendix

\section{Analysis on entity alignment task} \label{analysis appendix}
In this section, we analyze our proposed approach on the entity alignment task. For the entity alignment task, there exist two KGs, i.e., $\mathcal{G}_1 = (\mathcal{E}_1, \mathcal{R}_1, \mathcal{T}_1)$ and $\mathcal{G}_2 = (\mathcal{E}_2, \mathcal{R}_2, \mathcal{T}_2)$. Following the definitions in Section \ref{analysis}, we define the sets of relational prototype areas for $\mathcal{G}_1$ and $\mathcal{G}_2$ as $\mathcal{C}_1$ and $\mathcal{C}_2$, respectively. We assume that for a relational prototype area $C_{sr}$ in $\mathcal{C}_1$, its corresponding alignment relational prototype area $C_{tg}$ in $\mathcal{C}_2$ satisfies that $C_{sr}=C_{tg}$. 
Based on the loss function proposed by \eqref{loss function entity alignment}, the score function of entity alignment is 
\begin{align}\label{score function entity alignment}
    f(\textbf{e}_{sr}, \textbf{e}_{tg}) = - d(\textbf{e}_{sr}, \textbf{e}_{tg}) = -\lVert \textbf{e}_{sr} - \textbf{e}_{tg}\rVert, e_{sr}\in \mathcal{E}_1, e_{tg}\in \mathcal{E}_2.
\end{align}

Now, we want to prove that our approach can encourage the global semantic similarities of entities connected by the same relation on the entity alignment task.
However, it's hard to prove this property theoretically on the entity alignment task. The main difficulties are twofold. First, the activation function in \eqref{gcn rule new1} makes the aggregation operation in RPE-GCN be non-linear. Second, the update of an entity embedding in RPE-GCN depends on the embeddings of its neighbors (including entities and relational prototype entities). However, even entities connected by the same relation may have different neighbors, which makes their updates of embeddings diverse. Therefore, such properties make the theoretical analysis--the embeddings of entities connected by the same relation will become closer to each other---of our approach for entity alignment task difficult.





However, based on the experiments in Section \ref{performance of entity cluster}, a reasonable assumption is that the embeddings of entities connected by the same relation can become closer to each other on the entity alignment task. Then, we begin with this assumption and further get the following theorems.


\begin{theorem}\label{theorem_entity_alignment_1}
For an entity $e\in \mathcal{E}_1$, there exists a relational prototype area $C_{sr} \in \mathcal{C}_1$ such that $\textbf{e} \in C_{sr}$. Suppose that its corresponding alignment relational prototype area is $C_{tg}\in \mathcal{C}_2$, and the radius of $C_{tg}$ is $R$. If $d(C, C_{tg}) > 2 R$ for any $C \in \mathcal{C}_2 \setminus \{C_{tg}\}$, we have $f(\textbf{e}, \textbf{e}_1) > f(\textbf{e}, \textbf{e}_2)$ for any $\textbf{e}_1 \in C_{tg}, \textbf{e}_2\notin C_{tg}, e_1, e_2 \in \mathcal{E}_2$.
\end{theorem}

\begin{proof}
Since $C_{sr} = C_{tg}$, for any $\textbf{e}\in C_{sr}$, $\textbf{e}_2 \notin C_{tg}$, we have 
$$d(\textbf{e}_2, C_{tg}) > 2 R \Rightarrow d(\textbf{e}_2, C_{sr})> 2 R\Rightarrow d(\textbf{e}_2, \textbf{e}) > 2 R.$$
For any $\textbf{e}_1 \in C_{tg}$, we have $\textbf{e}_1 \in C_{sr}$ since $C_{sr} = C_{tg}$. Therefore, we can get
$$ 
d(\textbf{e}_1, \textbf{e}) \le 2 R < d(\textbf{e}_2, \textbf{e}).
$$
Therefore, we can get $f(\textbf{e}, \textbf{e}_1) > f(\textbf{e}, \textbf{e}_2)$ based on \eqref{score function entity alignment}, which completes the proof.
\end{proof}

\begin{theorem}\label{theorem_entity_alignment_2}
For an entity $e\in \mathcal{E}_2$, there exists a relational prototype area $C_{tg}\in \mathcal{C}_2$ such that $\textbf{e} \in C_{tg}$. Suppose that its corresponding alignment relational prototype area is $C_{sr}\in \mathcal{C}_1$, and the radius of $C_{sr}$ is $R$. If $d(C, C_{sr}) > 2 R$ for any $C \in \mathcal{C}_1 \setminus \{C_{sr}\}$, we have $f(\textbf{e}, \textbf{e}_1) > f(\textbf{e}, \textbf{e}_2)$ for any $\textbf{e}_1 \in C_{sr}, \textbf{e}_2\notin C_{sr}, e_1, e_2 \in \mathcal{E}_1$.
\end{theorem}

\begin{proof}
We omit the proof of Theorem \ref{theorem_entity_alignment_2} since it is similar to the proof of Theorem \ref{theorem_entity_alignment_1}. 
\end{proof}

Form Theorem \ref{theorem_entity_alignment_1} and Theorem \ref{theorem_entity_alignment_2}, we can see that the score of entities from  the category of the correct answer will be higher than entities from other categories, which benefits the performance.

\begin{table*}[ht]
    \caption{Results of ablation study on the DBP15k and DWY100k datasets, where ``w/o aggregation'' represents ``without aggregation'' proposed by \eqref{aggragation}.}
    \centering
    \resizebox{2.1\columnwidth}!{\begin{tabular}{lc c c c   c c c c  c c c c ccc }
        \toprule
          Methods&\multicolumn{3}{c}{\textbf{${\rm DBP_{ZH-EN}}$}}&  \multicolumn{3}{c}{\textbf{${\rm DBP_{JA-EN}}$}} & \multicolumn{3}{c}{\textbf{${\rm DBP_{FR-EN}}$}} &
          \multicolumn{3}{c}{\textbf{${\rm DBP-WD}$}} &
          \multicolumn{3}{c}{\textbf{${\rm DBP-YG}$}}\\
         \cmidrule(lr){2-4}
         \cmidrule(lr){5-7}
         \cmidrule(lr){8-10}
         \cmidrule(lr){11-13}
         \cmidrule(lr){14-16}
         & H@1 & H@10 & MRR & H@1 & H@10 & MRR  & H@1 & H@10 & MRR & H@1 & H@10& MRR & H@1 & H@10& MRR\\
        \midrule
        GCN (w/o aggregation) & 0.463&\textbf{0.833}&0.585&0.488&\textbf{0.852}&0.608&0.477&\textbf{0.862}&0.604&0.610    &\textbf{0.900}    &0.710    &0.745    &\textbf{0.941}    &0.813\\
        GCN & \textbf{0.477}&0.828&\textbf{0.593}&\textbf{0.495}&0.848&\textbf{0.613}&\textbf{0.498}&0.861&\textbf{0.619}&\textbf{0.611} &0.891    &\textbf{0.707}    &\textbf{0.747}    &0.939    &\textbf{0.814}\\
        \midrule
        RPE-GCN (w/o aggregation)&0.554&\textbf{0.884}    & {0.665}    & {0.561}    &\textbf{0.889}    & {0.673}    & {0.573}    &\textbf{0.906}    & {0.688}&{0.663}    &\textbf{0.923}    &{0.754}    & {0.789}    &\textbf{0.958}    & {0.849}\\
        RPE-GCN&\textbf{0.576}& {0.878}    &\textbf{0.678}    &\textbf{0.582}    & {0.883}    &\textbf{0.685}    &\textbf{0.597}    & {0.899}    &\textbf{0.701}& \textbf{0.678}    & {0.915}    & \textbf{0.762}    &\textbf{0.797}    & {0.957}    &\textbf{0.854}\\
        \bottomrule
    \end{tabular}}
    \label{ablation results}
\end{table*}

\section{More results of Long-tail Entities}\label{appendix c}
Figure~\ref{long-tail-appendix} shows the MRR results of long-tail entities on the DBP-WD and DBP-YG datasets. From the results, an interesting observation is that the MRR results of long-tail entities is better than others. However, the general trend is the same as the results in Section \ref{performance of long-tail}; that is, the smaller the max number of links, the more significant our model will improve the performance of the baselines. Recall that, if the max number of links is $N$, the entities have neighbors no larger than $N$.
\begin{figure}[H]
    \centering{
                \subfigure[DBP-WD]{ \label{task-adaptive-2}
            \includegraphics[height=0.5\columnwidth]{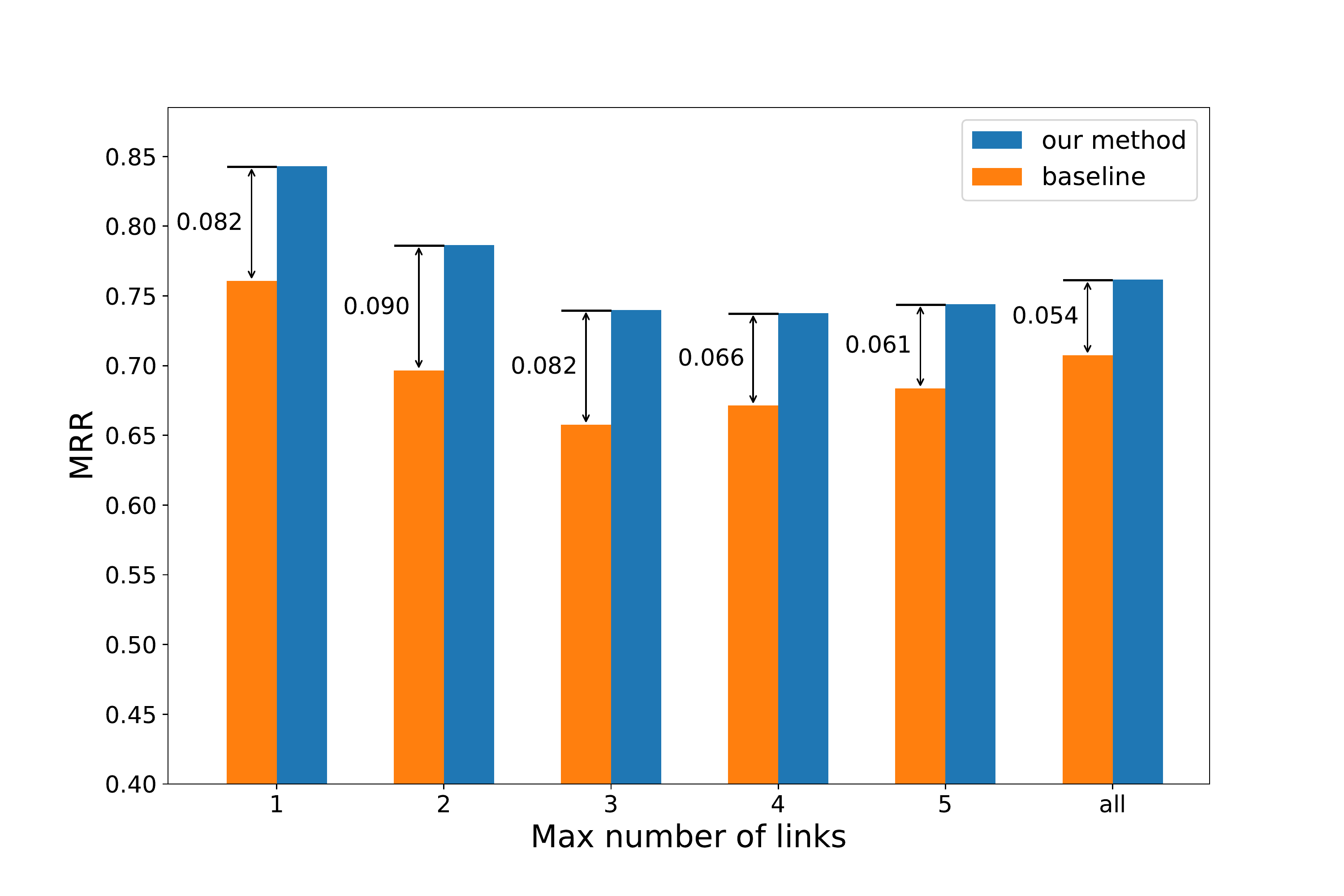}}
        \\
                \subfigure[DBP-YG]{ \label{task-adaptive-2}
            \includegraphics[height=0.5\columnwidth]{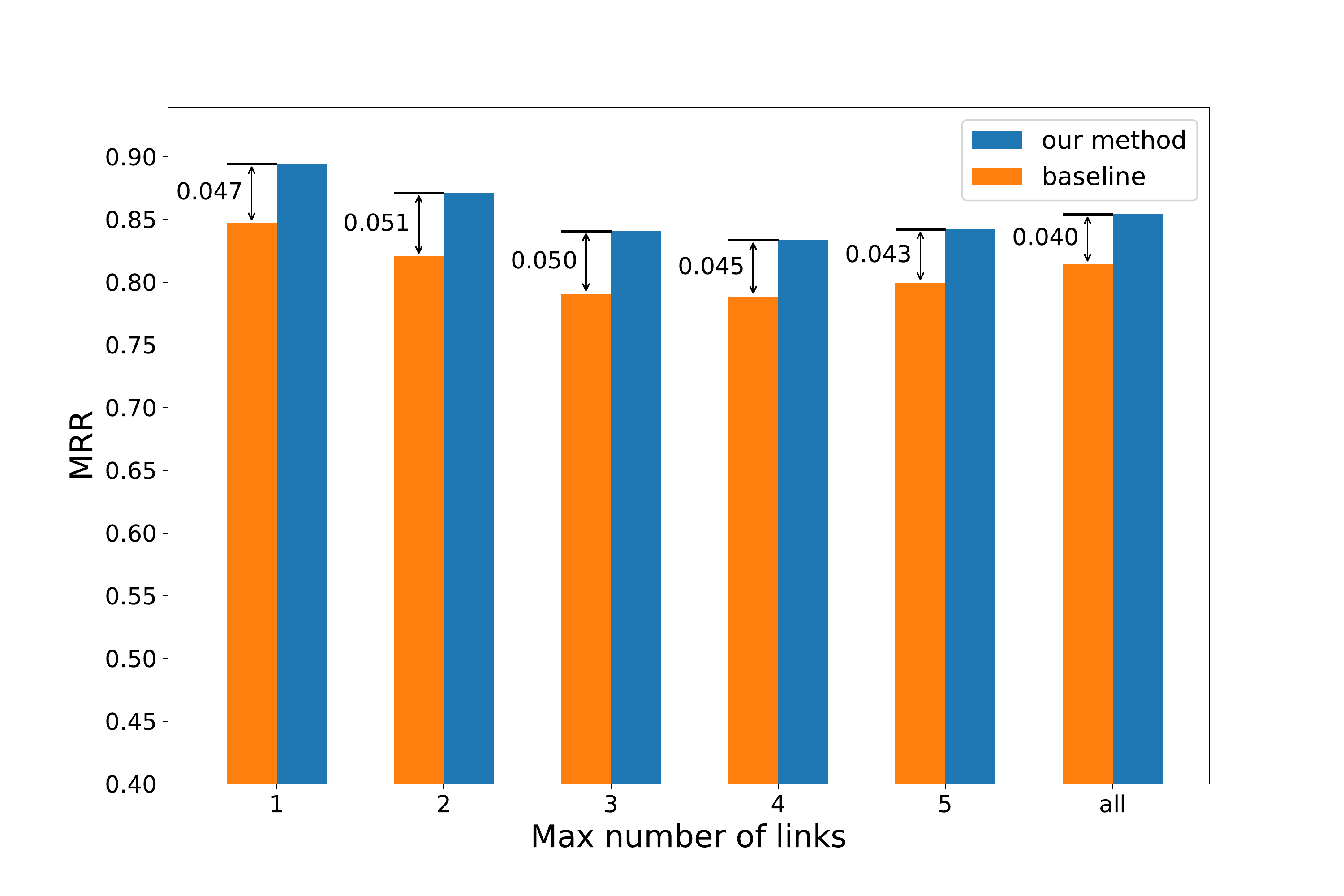}
        }}
    \caption{MRR results of entities w.r.t different number of links on DBP-WD and DBP-YG datasets.}
    \label{long-tail-appendix}
\end{figure}

\begin{table}[ht]
    \centering
    \caption{Hyperparameters of RotatE and RPE-RotatE. $k$ is the embedding size, $b$ is the batch size, $b'$ is the negative sampling size, $\gamma$ is the fixed margin in the loss function, $\alpha$ is the temperature of negative sampling, and $\lambda$ is the weight coefficient.}
\resizebox{1\columnwidth}!{\begin{tabular}{l *{7}{c}}
        \toprule
        Dataset & $k$ & $b$ & $b'$ & $\gamma$ &$\alpha$ &$\lambda$& learning rate\\
        \midrule
        WN18RR & 500 & 512 & 1024 & 6.0 &0.5 &0.5&0.00005\\
        FB15k-237 & 1000 & 1024 & 256 & 9.0 &1.0 &0.5&0.00005\\
        YAGO3-10 & 500 & 1024 & 400 & 24.0  &1.0&0.5&0.0002\\
        \bottomrule
    \end{tabular}}
    \label{hp:kg completion}
\end{table}

\begin{table}[ht]
    \centering
    \caption{Hyperparameters of GCN and RPE-GCN. $k$ is the embedding size, $\gamma$ is the fixed margin in the loss function, $L$ is the number of GCN layers, $\lambda$ is the weight coefficient, and $L2$ is the L2 regularizer weight.}
    \resizebox{1\columnwidth}!{\begin{tabular}{lccccccc}
        \toprule
        Dataset & $k$ & $\gamma$ & $L$ & $\lambda$& $L2$ & learning rate & dropout\\
        \midrule
         ${\rm DBP_{ZH-EN}}$ & 128& 1.0 &2&0.5&0.01&0.001&0.2\\
        ${\rm DBP_{JA-EN}}$ & 128& 1.0 &2&0.5&0.01&0.001&0.2\\
        ${\rm DBP_{FR-EN}}$ & 128& 1.0&2&0.5&0.01&0.001&0.2\\
        ${\rm DBP-WD}$& 128& 1.0&2&0.5&0.01&0.001&0.2\\
        ${\rm DBP-YG}$& 128& 1.0&2&0.5&0.01&0.001&0.2\\
        \bottomrule
    \end{tabular}}
    \label{hp:entity alignment}
\end{table}


\section{Ablation experiments}
In Table \ref{ablation results}, we show the ablation study of the aggregation strategy in \eqref{aggragation}. From the results, we find the strategy---which gets the entity embedding  by aggregating the hidden states of all the layers---can boost the performance on H@1 and MRR for both GCN and RPE-GCN models. It demonstrates the effectiveness of the used strategy. Also, we can see that the used strategy can only obtain comparable performance on H@10 to the strategy that only uses the hidden outputs at the last layer,
which means that our used strategy mainly increase the percentage of the targets that  have been correctly ranked first.



\section{Hyperparameters}\label{hyperparameters}
We list the hyperparameters of our proposed methods on the entity alignment and KG completion tasks in Tables  \ref{hp:entity alignment}  and \ref{hp:kg completion}, respectively. To make a fair comparison, we set the hyperparameters of our models and the baseline models to the same. We use Adagrad \cite{adagrad} as the optimizer on the entity alignment task, and use Adam \cite{adam} as the optimizer on the KG completion task.
Notice that, the best hyperparameter settings of GCN and RotatE are taken from MuGNN \cite{mugnn} and RotatE \cite{rotate}, respectively.

\end{document}